\documentclass[letterpaper]{article}

\PassOptionsToPackage{numbers, compress}{natbib}

\usepackage[final]{neurips_2019}

\usepackage{amssymb}
\usepackage{amsmath}
\usepackage{amsthm}
\usepackage{dsfont}

\usepackage[utf8]{inputenc} 
\usepackage[T1]{fontenc}    
\usepackage{url}            
\usepackage{booktabs}       
\usepackage{nicefrac}       
\usepackage{microtype}      

\usepackage{xspace}

\usepackage{wasysym}
\usepackage{subfigure}
\usepackage{wrapfig, blindtext}

\usepackage{soul}
\newif\ifSLIDES
\SLIDESfalse

\input{macros-envs.sty}
\input{macros-comments.sty}
\input{macros-code.sty}
\input{macros-symbols.sty}

\input{macros-probas-stats-information-theory-al.sty}
\input{./macros-local-v2.sty}

\newif\ifLONG
\LONGfalse

\DeclareMathOperator{\minuslast}{^{-1}}

\title{Low-Complexity Nonparametric Bayesian \\ Online Prediction with Universal Guarantees}

\author{
	Alix Lh\'eritier\\
	Amadeus SAS\\
	F-06902 Sophia-Antipolis, France\\
	\texttt{alix.lheritier@amadeus.com}
	\And
	{Fr\'ed\'eric Cazals}\\
        Universit\'e C\^ote d'Azur\\
        Inria\\
	F-06902 Sophia-Antipolis, France\\
	\texttt{frederic.cazals@inria.fr}
}

\begin{document}

\maketitle

\begin{abstract}

We propose a novel nonparametric online predictor for discrete
labels conditioned on multivariate continuous features.
The predictor is based on a feature space discretization induced by a 
full-fledged k-d tree with randomly picked directions and a recursive Bayesian distribution, which allows to automatically learn the most relevant feature scales characterizing the conditional distribution.
We prove its pointwise universality, i.e., it achieves a normalized log loss performance asymptotically as good as the true conditional entropy of the labels given the features.
The time complexity to process the $n$-th sample point is $\bigO{\log n}$ in probability with respect to the distribution generating the data points, whereas other exact nonparametric methods require to process all past observations.
%
Experiments on challenging datasets show the computational and statistical efficiency of our algorithm in comparison to standard and state-of-the-art methods. 
\end{abstract}

\newcommand{\myparagraph}[1]{\vspace{-.1cm} \noindent{\bf #1}}

\section{Introduction}

\myparagraph{Universal online predictors.}
An \emph{online (or sequential) probability predictor} 
processes sequentially input symbols
$\rvl_1,\rvl_2,\dots$ belonging to some alphabet $\alphabet$. Before 
observing the next symbol in the sequence, it predicts it by
estimating the probability of observing each symbol of the
alphabet. Then, it observes the symbol and some loss is incurred
depending on the estimated probability of the current
symbol. Subsequently, it adapts its model in order to better predict
future symbols.  The goal of \emph{universal prediction} is to achieve
an asymptotically optimal performance independently of the generating
mechanism (see, e.g., the survey of Merhav and Feder
\cite{merhav1998universal}). When performance is measured in terms of
the logarithmic loss, prediction is intimately related to data
compression, gambling and investing (see, e.g.,
\cite{cover2006elements,cesa2006prediction}).

Barron's theorem \cite{Barron1998information} (see also
\cite[Ch.15]{grunwald2007minimum}) establishes a fundamental link
between prediction under logarithmic loss and learning: the better we
can sequentially predict data from a probabilistic source, the faster
we can identify a good approximation of it. This is of paramount
importance when applied to nonparametric models of infinite
dimensionality, where overfitting is a serious concern.
This is our case, since the predictor observes some associated side
information (i.e. \emph{features}) $\rvp_i\in\Rd$ before predicting
$\rvl_i\in\alphabet$, where $\alphabet = \{
\lambda_1,\dots,\lambda_{\size{\alphabet}}\}$. We consider the
probabilistic setup where the pairs of observations
$(\rvpi{i},\rvli{i})$ are i.i.d.~realizations of some random variables
$(\rvP,\rvL)$ with joint probability measure $\probadd$. Therefore, we
aim at estimating a nonparametric model of the conditional
measure $\probaid{\rvL|\rvP}$.

Nonparametric distributions can be approximated by universal
distributions over countable unions of parametric models (see e.g.,
\cite[Ch.~13]{grunwald2007minimum}). This approach requires defining
parametric models that can arbitrarily approximate the nonparametric
distribution as the number of parameters tend to infinity. For
example, models based on histograms with arbitrarily many bins have
been proposed to approximate univariate nonparametric densities (e.g.,
\cite{hall1988stochastic,rissanen1992density,yu1992data}).

Bayesian mixtures allow to obtain universal distributions for countable unions of parametric models (e.g., \cite{willems1995context,willems1998context}). 
Nevertheless, standard Bayesian mixtures suffer from the catch-up phenomenon,
i.e., their convergence rate is not optimal. In
\cite{erven2012catching}, it has been shown that a better convergence
rate can be achieved by allowing models to change over time, by
considering, instead of a set of distributions ${\cal M}$, a (larger)
set constituted by sequences of distributions of ${\cal M}$. 
The resulting \emph{switch distribution} has
still a Bayesian form but the mixture is done over sequences of
models.
%
%
%

Previous works on prediction with side information either are
non-sequential (e.g. PAC learning \cite{valiant1984theory}), or use
other losses (e.g. \cite{gyofi2005strategies,gyorfi1999simple} ) or
consider side information in more restrictive spaces
(e.g. \cite{algoet1992universal,cai2005universal}).  
Our work bears similarities to \cite{kozat2007universal,tziortziotis2014cover,veness2017online} but the objectives are different and so are the guarantees.
Recently, \cite{lheritier2018sequential} proposed a universal online predictor
for side information in $\Rd$ based on a mixture of nearest-neighbors
regressors with different $k(n)$ functions specifying the number of
neighbors at time $n$. Practically, the performances depend on the
particular set of functions---a design choice---and its time
complexity is linear in $n$ due to the exact nearest neighbor
search. 
Gaussian Processes (see, e.g., \cite{rasmussen2005}) are
  nonparametric Bayesian methods which can be used for online
  prediction with side information. It is conjectured that exact
  Gaussian processes with the radial basis function (RBF) kernel are
  universal under some conditions on the marginal measure
  $\probaid{\rvP}$ \cite[Sec.~13.5.3]{grunwald2007minimum}.  In
  practice, approximations are required to compute the predictive
  posterior for discrete labels (e.g.~Laplace) and the kernel width
 strongly affects the results.  In addition, their time
  complexity to predict each observation is $\bigO{n^3}$, making them
  practical for small samples sizes only.

We propose a novel nonparametric online predictor with universal
guarantees for continuous side information exhibiting two distinctive
features.  First, it relies on a hierarchical feature space
discretization and a recursive Bayesian distribution, automatically
learning the relevant feature scales and making it
scale-hyperparameter free.  Second, in contrast to other nonparametric
approaches, its time complexity to process the $n$-th sample point is
$\bigO{\log n}$ in probability.
%
%
Due to space constraints, proofs are presented in 
the  supplementary material.

\section{Basic definitions and notations}
In order to represent sequences, we use the notation $x^n\equiv x_1,\dots,x_n$. 
The functions $\nbSamples{\cdot}$ and $\nbLabels{\cdot}{\lambda}$, give, respectively, the length of a sequence and 
the number of occurrences of a symbol $\lambda$ in it.
%
Let $\probadd$ be the joint probability measure of 
$\rvL,\rvP$. 
Let $\probaid{\rvL},\probaid{\rvP}$ be their respective marginal measures 
and $\probaid{\rvP|\rvL}$ the probability measure of $\rvP$ conditioned on $\rvL$. 
%
%
%
The entropy of random variables is denoted $\entr{\cdot}$, while the
entropy of $\rvL$ conditioned on $\rvP$ is denoted
$\centri{\rvL}{\rvP}$. The mutual information between $\rvL$
and $\rvP$ is denoted $\mutinfo{\rvL}{\rvP}$. Logarithms are taken in base 2.

A \emph{finite-measurable partition} $A = (\Celi{1}, \dots , \Celi{n})$ of some set $\Omega$ is a subdivision of $\Omega$ 
into a finite number of disjoint measurable sets or \emph{cells} $\Celi{i}$ whose union is $\Omega$. 
An $n$-sample partition rule $\pi_n(\cdot)$ is a mapping from $\Omega^n$ 
to
the space of finite-measurable partitions for $\Omega$, denoted ${\cal A}(\Omega)$. 
A partitioning scheme for $\Omega$ is a countable
collection of 
$k$-sample partition rules $\Pi \equiv \{\pi_k\}_{k\in\mathbb{N}^+}$.
The partitioning scheme at time $n$ defines the set of partition rules 
$\Pi_n\equiv \{\pi_k\}_{k=1..n}$. 
For a given $n$-sample partition rule $\pi_n(\cdot)$ and a sequence
$\seqzN\in\Omega^n$, $\pi_n(z|\seqzN)$ denotes
the unique cell in $\pi_n(\seqzN)$ containing a given point $z\in\Omega$.
For a given partition $A$, let $A(z)$  denote the unique cell of $A$ containing $z$.
Let $\Celdof{\cdot}$ denote the operator that extracts the subsequences whose symbols have corresponding $\rvp_i\in\Celd$.

\toblack

\section{The \chronoswitch distribution} 

\newcommand{\figCellCreationContent}{
	\centering
	\subfigure[$\pi_1(z^1)$: $\Celi{1}$ and $\Celi{2}$ are created, $z_1\in \Celi{1}$. ]{\includegraphics[width = .2\columnwidth] {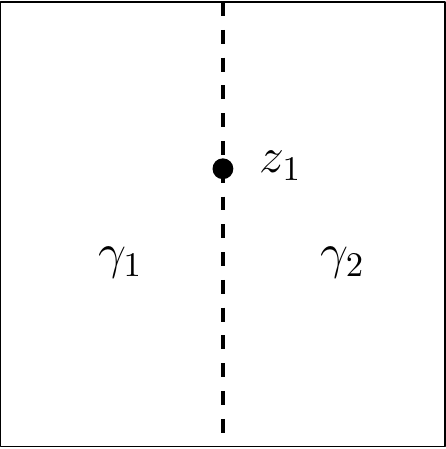}}~~ 
	\subfigure[$\pi_2(z^2)$: $\Celi{2,1}$ and $\Celi{2,2}$ are created. $z_2\in \Celi{2,2}$.]{\includegraphics[width = .2\columnwidth] {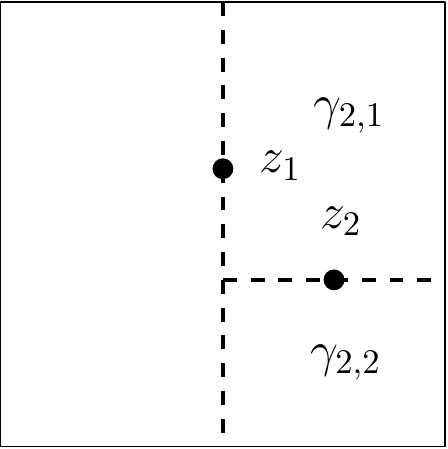}}
\caption{{\bf Cell creation process and cell splitting index}. The
  cell splitting index is defined w.r.t.~its subsequence:
  $\tau_2(\Omega)=1,\tau_2(\Celi{2})=1$ since $\Celiof{2}{z^2}=z_2$,
  and $\tau_2(\Celi{1})=\tau_2(\Celi{2,1})=\tau_2(\Celi{2,2})=\infty$.
}
\label{fig:creation-process} 
}

We define the \chronoswitch distribution $\estimPWithKDSi$ using a k-d tree based hierarchical partitioning and a switch distribution defined over the union of multinomial distributions implied by the partitioning.

\myparagraph{Full-fledged k-d tree based spatial partitioning.}
We obtain a hierarchical partitioning of $\Omega=\Rd$ using a full-fledged 
k-d tree \citep[Sec. 20.4]{devroye1996probabilistic} that is naturally amenable to an online construction since pivot points are chosen in the same order as sample points are observed.
Instead of rotating deterministically the axis of the projections, we sample the axis uniformly at each node of the tree. 
Formally, let $\Pi_\text{kd}\equiv \{\pi_k\}_{k\in\mathbb{N}^+}$ be the nested
partitioning scheme such that $\pi_n(z^n)$ is the spatial partition
generated by a full-fledged k-d tree after observing $z^n$.  In order
to define it recursively, let the base case be $\pi_0(z^0)\equiv\Rd$,
where $z^0$ is the empty string.  Then, $\pi_{n+1}(z^{n+1})$ is
obtained by uniformly drawing a direction $J$ in $1..d$ and by
replacing the cell $\Celd\in\pi_n(z^n)$ such that $z_{n+1}\in \Celd$
by the following cells
\begin{equation}
\label{eq:kd-split}
\begin{cases}
\Celi{1} \equiv \{z\in\Celd: z[J] \leq z_{n+1}[J] \} \\
\Celi{2} \equiv \{z\in\Celd: z[J] > z_{n+1}[J] \} 
\end{cases}
\end{equation}
where $\cdot[J]$ extracts the $J$-th coordinate of the given vector.
A spatial partition $A=\left\{
\Celi{1},\Celi{2},\dots,\Celi{\size{A}} \right\}$ of $\Rd$ defines a class of 
piecewise multinomial distributions characterized by 
$\thetaVector{A} \equiv
[\theta_1,\dots,\theta_{\size{A}}],\;\theta_i\in \Delta^{\size{\alphabet}},$
where $\Delta^{\size{\alphabet}}$ is the standard $\size{\alphabet}$-simplex. 
More precisely, $P_{\thetaVector{A}}(\cdot|\rvp)$ is multinomial with parameter 
$\theta_i$ if $\rvp\in\Celi{i}$.


\paragraph{Context Tree Switching.}
\begin{wrapfigure}{r}{0.5\linewidth}
\figCellCreationContent
\end{wrapfigure}
We adapt the Context Tree Switching (CTS) distribution \cite{veness2012context}  to use spatial cells as contexts. Since these contexts are created as sample points $z_i$ are observed, the chronology of their creation has to be taken into account.
Given a nested partitioning scheme $\Pi$ whose instantiation with $z^n$ creates a cell $\Celd$ and splits it into $\Celi{1}$ and $\Celi{2}$, we define the \emph{cell splitting index} $\tau_n(\Celd)$ as the index in the subsequence $\Celdof{\rvp^n}$ when $\Celi{1}$ and $\Celi{2}$ are created (see Fig.~\ref{fig:creation-process}). If $\Celd$ is not split by $\Pi$ instantiated with $z^n$, then we define $\tau_n(\Celd)\equiv\infty$.

At each cell $\Celd$, two models, defined later and denoted $a$ and $b$, are considered.
Let $w_{\Celd}(\cdot)$ be a prior over model index sequences 
$\seq{i}{m}\equiv 
i_1,\dots,i_{m} 
\in \{a,b\}^{m}$ at cell $\Celd$, 
recursively defined by 
\begin{equation*}
\label{eq:cts-prior}
w_{\Celd}(\seq{i}{m})\equiv \\
\begin{cases}
1 \text{ if }m=0 \\
\frac{1}{2} \text{ if }m=1 \\ 
w_{\Celd}(\seq{i}{m-1})\left( (1-\alpha_m^{\Celd}) \mathds{1}_{E}   
 +  \alpha_m^{\Celd} \mathds{1}_{\neg E} \right) \text{ if } m>1 
\end{cases}
, E\equiv\{i_{m}=i_{m-1}\}, \alpha_m^\Celd=m^{-1}.
\end{equation*}

In order to define the CTS distribution, we need the Jeffreys' mixture
over multinomial distributions also known as the Krichevsky-Trofimov
estimator \cite{trofimov1981performance} 
\begin{equation}
\estimPWithBayesJeffreys{\Celd}{}{\seqlN} \equiv
\int_{\theta \in \Delta^{\size{\alphabet}} }  \prod_{j\in 1\dots\size{\alphabet}}
\theta[j]^{\nbLabels{\seqlN}{\lambda_j}}   
w(\theta)d\theta
\end{equation}
with $\theta[j]$ being the $j$-th component of the vector $\theta$, $\nbLabels{\seqlN}{\lambda_j}$ the number of occurrences of $\lambda_j$ in $\seqlN$  and $w(\cdot)$ the Jeffreys' prior for the 
multinomial distribution \cite{jeffreys1946invariant}
i.e.~a Dirichlet distribution with parameters $(\nicefrac{1}{2},\dots,\nicefrac{1}{2})$.

%

Consider any cell $\Celd$ created by the partitioning scheme $\Pi$ instantiated with $z^n$.
$\Celd$ can either be refined into two child cells
$\Celi{1}$ and $\Celi{2}$ or have $\tau_n(\Celd)=\infty$. 
Given a sequence of labels $\seqlN$ such that all the corresponding positions $\rvp_i\in \Celd$,
the modified CTS distribution is given by
\begin{equation}
\label{eq:cts-def}
\estimPWithCTSi{\Celd}{\Pi}{\seqlcondzN} 
\equiv \\ 
\sum_{\seq{i}{n}\in \{a,b\}^{n}} 
w_\Celd(\seq{i}{n}) \prod_{k=1}^{n}
\left[
\mathds{1}_{\{i_k=a \}} 
\phi_a (\rvl_k | \seq{\rvl}{k-1} ) \right. 
\left. +  
\mathds{1}_{\{i_k=b \}} 
\phi_b^\Celd (\rvl_k | \seq{\rvl}{k-1},\seq{\rvp}{k} )
\right]
\end{equation}
where the predictive distributions of models $a$ and $b$ are given by 
\begin{equation}
\label{eq:model-a}
\phi_a (\rvl_k | \seq{\rvl}{k-1} ) \equiv \estimPWithBayesJeffreys{\Celd}{}{\rvl_k|\seq{\rvl}{k-1}} \equiv \frac{\estimPWithBayesJeffreys{\Celd}{}{\seq{\rvl}{k}}}{\estimPWithBayesJeffreys{\Celd}{}{\seq{\rvl}{k-1}}}
\end{equation} 
and 
\begin{equation}
\label{eq:model-b}
\phi_b^\Celd (\rvl_k | \seq{\rvl}{k-1},\seq{\rvp}{k} ) \equiv 
\begin{cases}
 \estimPWithBayesJeffreys{\Celd}{}{\rvl_k|\seq{\rvl}{k-1}} \text{ if }  k<\tau_k(\Celd)  \\ 
\frac{\estimPWithCTSi{\Celi{j}}{\Pi}{\Celiof{j}{\seq{\rvl}{k}}
		\vert\Celiof{j}{\seq{\rvp}{k} }
}}
{\estimPWithCTSi{\Celi{j}}{\Pi}{ \Celiof{j}{\seq{\rvl}{k}}\minuslast
		\vert\Celiof{j}{\seq{\rvp}{k} }\minuslast
}}
 \text{ with } j: \rvp_k\in\Celi{j} \text{, otherwise} 
\end{cases}
\end{equation}
where $\cdot\minuslast$ removes the last symbol of a sequence and, for the empty sequences $l^0,z^0$, $\estimPWithCTSi{\Celd}{\Pi}{l^0\vert z^0}\equiv 1$ and $\estimPWithBayesJeffreys{\Celd}{}{\seq{\rvl}{0}}\equiv 1$.


\myparagraph{Definition of $\estimPWithKDSi$.}
The \chronoswitch distribution is obtained from the modified CTS distribution on the context cells defined by the full-fledged k-d tree spatial partitioning scheme i.e. 
\begin{equation}
\estimPWithKDSi{\seqlcondzN} \equiv \estimPWithCTSi{\Rd}{\Pi_\text{kd}}{\seqlcondzN} 
.
\end{equation}

\begin{remark}
	In \cite{veness2012context}, the authors observe better empirical performance with 
	$\alpha_m^{\Celd}=n^{-1}$ for any cell $\Celd$, where $n$ is the number of samples observed at the root partition $\Omega$ when the $m$-th sample is observed in $\Celd$. 
	With this switching rate they were able to provide a good redundancy bound for bounded depth trees. In our unbounded case, we observed a better empirical performance with  $\alpha_m^\Celd=m^{-1}$. 
\end{remark}

\begin{remark}
	\label{rmk:kdw-def}
	A Context Tree Weighting \cite{willems1995context} scheme can be
	obtained by setting $\alpha_m^\Celd=0$.
	The corresponding distribution is denoted $\estimPWithKDWi$.
\end{remark}

\section{Pointwise universality}

In this section, we show that $\estimPWithKDSi$ is pointwise universal, i.e.~it achieves a normalized log loss asymptotically as good as the true conditional entropy of the source generating the samples.
More formally, we state the following theorem. 
\begin{theorem}
	\label{thm:chrono-universal}
	The \chronoswitch distribution is pointwise universal, i.e.
	\begin{equation}
	- \limn \frac{1}{n}\log \estimPWithKDSi{\seqLcondZN} 
	\leq \centri{\rvL}{\rvP} \as  
	\end{equation}
	for any probability measure $\probadd$ generating the samples
        such that $\probaid{\rvP|\rvL}$ are absolutely
          continuous with respect to the Lebesgue measure.
\end{theorem}

In order to prove Thm.~\ref{thm:chrono-universal}, we first show that $\estimPWithCTSi{\Omega}{\Pi}$ is universal with respect to the class of piecewise multinomial distributions defined by any nested partitioning scheme $\Pi$. Then, we show that $\Pi_\text{kd}$ allows to approximate arbitrarily well any conditional distribution.

\myparagraph{Universality with respect to the class of piecewise multinomial distributions.}
Consider a nested partitioning scheme $\Pi$ for $\Omega$. 
$\Pi_n$ instantiated with some $z^n\in\Omega^n$ naturally defines a tree structure whose root node represents $\Omega$.
Given an  arbitrary set of internal nodes, 
we can prune the tree by transforming these internal nodes into leaves and  discarding the corresponding subtrees. 
The new set of leaf nodes define a partition of $\Omega$.
Let ${\cal P}_n(z^n)$ be the set of all the partitions that can be obtained by pruning the tree induced by $\Pi_n$ instantiated with $z^n$.

The next lemma shows that $\estimPWithCTSi{\Omega}{\Pi}$, defined in Eq.~\ref{eq:cts-def}, is universal 
with respect to the class of piecewise multinomial distributions defined on the 
partitions ${\cal P}_n(z^n)$.


\begin{lemma}
	\label{lem:ctw-twice-universal}
	Consider arbitrary sequences $l^n\in\alphabet^n,z^n\in\Omega^n, n\geq0$. Then, for any $A\in{\cal P}_{n}(z^n)$ and 	
	for any piecewise multinomial 
	distribution $P_{\thetaVector{A}}$, 
	the following holds
	\begin{equation}
	\label{eq:cts-bound}
	- \log \estimPWithCTSi{\Omega}{\Pi}{\seqlcondzN} \leq  
	- \log P_{\thetaVector{A}}(\seqlcondzN) 
	+ |A|\zeta\left(\frac{n}{|A|}\right) 
	+ \Gamma_A \log 2n
	+ \bigO{1}
	\end{equation}
	and 
	\begin{equation}
	- \limn \frac{1}{n}\log \estimPWithCTSi{\Omega}{\Pi}{\seqLcondZN} \leq 
	\centri{\rvL}{A(\rvP)} \as
	\end{equation}
	where $\Gamma_A$ is the number of nodes in the tree, induced by $\Pi_n$ instantiated with $z^n$, 
	that represents $A$ (i.e., the code length given by a \emph{natural} code for unbounded trees) and 
	\begin{equation}
	\zeta(x)\equiv 
	\begin{cases}
	x \log \size{\alphabet} & \text{ if } 0\leq x < 1 \\
	\frac{\size{\alphabet}-1}{2}\log x + \log \size{\alphabet} & \text{ if } x\geq 1 \\
	\end{cases}
	.
	\end{equation}
\end{lemma}


\begin{remark}
In a Context Tree Weighting scheme ($\alpha_m^\Celd=0$), the $\log 2n$ factor in Eq.~\ref{eq:cts-bound} disappears. See proof of Lemma \ref{lem:ctw-twice-universal}. 
Thus, universality holds for this case too.
\end{remark}

\myparagraph{Universal discretization of the feature space.}
In order to prove that the k-d tree based partitions allow to
approximate arbitrarily well the conditional entropy
$\centri{\rvL}{\rvP}$, we use the following corollary of
\cite[Thm. 4.2]{silva2008optimal}. 
\begin{corollary}
\label{cor:silva}
Let $\diameter(\Celd)\equiv\sup_{x,y\in \Celd}\vvnorm{x-y}$. Let
$\probadd$ be any probability measure such that
  $\probaid{\rvP|\rvL}$ are absolutely continuous with respect to the
  Lebesgue measure.  Given a partition scheme $\Pi \equiv \{\pi_k\}_{k\in\mathbb{N}^+}$, 
if $\forall \delta>0$
	\begin{equation}
	\label{eq:shrinking-cond}
	\probai{\rvP}{\left\{ \rvp\in \Rd:\diameter(\pi_n(\rvp|\seqZN))>\delta 
		\right\}}
	\toas 0   
	\end{equation}
	then $\Pi$ universally discretizes the feature space, i.e.
	\begin{equation}
	\centri{\rvL}{\pi_n(\rvP|\seqZN)} \toas \centri{\rvL}{\rvP} .
	\end{equation}
\end{corollary}	


The next lemma provides the required shrinking condition for the k-d tree based partitioning.
\begin{lemma}
\label{lem:k-d-tree}
$\Pi_\text{kd}$ satisfies the shrinking condition 
of Eq.~\ref{eq:shrinking-cond} and, thus, universally discretizes the feature space.
\end{lemma}



\myparagraph{Pointwise universality.} 
The proof of Thm.~\ref{thm:chrono-universal}
on the pointwise universality of $\estimPWithKDSi$ 
stems from a combination of Lemmas~\ref{lem:ctw-twice-universal} and \ref{lem:k-d-tree}---see Appendix.

\section{Online algorithm}

Since a direct computation of Eq.~(\ref{eq:cts-def}) is intractable
and an online implementation is desired, we use the recursive form of
\cite[Algorithm 1]{veness2012context}, which performs the exact
calculation.  We denote by $\estimPWithSEQu{\Celd}$ the sequentially
computed \chronoswitch distribution at node $\Celd$.  In
Section~\ref{sec:correctness}, we show that
$\estimPWithSEQi{\Celd}{\seqlcondzN} =
\estimPWithCTSi{\Celd}{\Pi_\text{kd}}{\seqlcondzN}$.

\subsection{Algorithm}

\myparagraph{Outline.}
For each node of the k-d tree, the algorithm maintains 
two weights denoted $w_\Celd^a$ and $w_\Celd^b$.
As follows from \cite[Lemma 2]{veness2012context}, if $\rvl^t,\rvp^t$ are the subsequences observed in $\Celd$ and  $w_\Celd^a$ is the weight before processing $\rvl_t$, then, $w_\Celd^a \estimPWithBayesJeffreys{\Celd}{}{\rvl_t|\seq{\rvl}{t-1}}$  corresponds to the contribution of all possible model sequences ending in model $a$ (KT) to the total probability assigned to $\rvl^{t}$ by the CTS distribution.
Analogously, $w_\Celd^b \estimPWithSEQu{\Celd}{\rvl_t|\rvl^{t-1}, \rvp^t}$  corresponds to the contribution of all possible model sequences ending in model $b$ (CTS).

We now describe the three steps that allow the online computation of $\estimPWithSEQu{\Celd}{\seqlcondzN}$ given by Eq.~\ref{eq:final-estimate}. 
The algorithm starts with only a root node representing $\Rd$. When a new point $\rvp_*$ is observed, the following steps are performed.

\myparagraph{Step 1: k-d tree update and new cells' initialization.} The point  $\rvp_*$ is passed down the k-d tree until it reaches a leaf cell $\Celd$.
Then, a coordinate $J$ is uniformly drawn from $1\dots d$ and two child nodes, corresponding to the new cells defined in Eq.~\ref{eq:kd-split} with $z_*$ as splitting point, are created. 

Let $\rvl^n,\rvp^n$ be the subsequences observed in $\Celd$ and thus $\rvp_n=\rvp_*$.
Since the new cells may contain some of the 
symbols in $\seq{\rvl}{n-1}$, the following initialization is 
performed at each new node  $\Celi{i}, i\in\{1,2\}$:
\begin{align}
\label{eq:weight-init}
\begin{split}
&w_{\Celi{i}}^a \leftarrow \frac{1}{2} 
\estimPWithBayesJeffreys{\Celi{i}}{}{\Celiof{i}{\seq{\rvl}{n-1}}}\\     
&w_{\Celi{i}}^b \leftarrow \frac{1}{2} 
\estimPWithBayesJeffreys{\Celi{i}}{}{\Celiof{i}{\seq{\rvl}{n-1}}}     
\end{split}
\text{, with }\estimPWithBayesJeffreys{\Celi{i}}{}{\Celiof{i}{\seq{\rvl}{n-1}}}=1   
\text{ if }\Celiof{i}{\seq{\rvl}{n-1}}\text{ is empty}.
\end{align}
 
\myparagraph{Step 2: Prediction.}
The probability assigned to the subsequence $\seqlN$ given $\seqzN$ observed in $\Celd$ is 
\begin{equation}
\label{eq:final-estimate}
\estimPWithSEQu{\Celd}{\seqlcondzN} \leftarrow 
w_\Celd^a \estimPWithBayesJeffreys{\Celd}{}{\rvl_n|\rvl^{n-1}} + 
w_\Celd^b\estimPWithRECu{\Celd}{\rvl_n|\rvl^{n-1},\rvp^n}
\end{equation}
where 
\begin{equation}
\estimPWithRECu{\Celd}{\rvl_n|\rvl^{n-1},\rvp^n} \leftarrow 
\begin{cases} 
\estimPWithBayesJeffreys{\Celd}{}{\rvl_n|\seq{\rvl}{n-1}} \text{ if } n < \tau_n(\Celd) \\
\frac{\estimPWithSEQu{\Celi{j}}{\Celiof{j}{\seqlN}\vert \Celiof{j}{\seqzN} } }
{\estimPWithSEQu{\Celi{j}}{\Celiof{j}{\seqlN}\minuslast \vert \Celiof{j}{\seqzN}\minuslast}} 
\text{ with } j:\rvpi{n} \in \Celi{j},\text{ otherwise}
\end{cases}
.
\end{equation} 


\myparagraph{Step 3: Updates.}
Having computed  the probability assignment of Eq.~\ref{eq:final-estimate}, the weights of the nodes corresponding to the cells $\{\Celd: \rvp_*\in \Celd\}$
are updated. Given a node $\Celd$ to be updated, let $\rvl^n,\rvp^n$ be the subsequences observed in $\Celd$.
The following updates are applied:
\begin{align}
	\label{eq:weight-update2}
	\begin{split}
		&w_\Celd^a \leftarrow \alpha^\Celd_{n+1} \estimPWithSEQu{\Celd}{\seqlcondzN} +
		\beta^\Celd_{n+1} w_\Celd^a 
		\estimPWithBayesJeffreys{\Celd}{}{\rvl_n|\seq{\rvl}{n-1}} \\
		&w_\Celd^b \leftarrow \alpha^\Celd_{n+1} \estimPWithSEQu{\Celd}{\seqlcondzN} +
		\beta^\Celd_{n+1} w_\Celd^b 
		\estimPWithRECu{\Celd}{\rvl_n|\rvl^{n-1},\rvp^n}
	\end{split}
\end{align}
where $\beta^\Celd_n \equiv (1 - 2\alpha^\Celd_n)$. When $\Celd$ has just been created (i.e. $\Celd$ is a leaf node), these updates reduce to
\begin{align}
	\label{eq:weight-update}
	\begin{split}
		&w_\Celd^a \leftarrow w_\Celd^a \estimPWithBayesJeffreys{\Celd}{}{\rvl_n|\seq{\rvl}{n-1}} \\
		&w_\Celd^b \leftarrow w_\Celd^b \estimPWithBayesJeffreys{\Celd}{}{\rvl_n|\seq{\rvl}{n-1}} 
	\end{split}
	.
\end{align}


%

%


\begin{remark}
	\label{rmk:KT-seq}
The KT estimator can be computed sequentially using the following
formula~\cite{shtar1987universal}:
	\begin{equation}
	\estimPWithBayesJeffreys{}{}{\rvl_n|\seq{\rvl}{n-1}} =
	\frac{\nbLabels{\rvl^{n-1}}{\rvli{n}}+\frac{1}{2}} 
		{\nbSamples{\rvl^{n-1}} + \frac{\size{\alphabet}}{2}}
		.
	\end{equation}
%
	Therefore, the sequential computation only requires  maintaining the 
	counters 
	$\nbLabels{\rvl^{n-1}}{\rvli{n}}$ for each cell.
\end{remark}
\begin{remark}
	 Samples $z_i$ only need to be stored at leaf nodes. Once a leaf node is split, they are moved to their corresponding child nodes.
\end{remark}

\subsection{Correctness}
\label{sec:correctness}

The steps of our algorithm are the same as those of \cite[Algorithm 1]{veness2012context} except for the  initialization of Eq.~\ref{eq:weight-init}.
In fact, as shown in the next lemma, it is equivalent to building the partitioning tree from the beginning (assuming, without loss of generality, that $\seqzN$ is known in advance) and applying the original algorithm at every relevant context. 

\begin{lemma}
\label{lem:delayed-equiv}
Let $n\in\mathbb{N}^+$ and assume the partitioning tree for $\rvp^n$ is built from the beginning. Let $\Celd$ be any node of the tree.
 If the original initialization and update equations from \cite[Algorithm 1]{veness2012context} (corresponding respectively to Eq.~\ref{eq:weight-init} with an empty sequence and Eq.~\ref{eq:weight-update2}) are applied, the weights, after observing $\rvl^t$ in $\Celd$ with  $t<\tau_n(\Celd)$, are $w_\Celd^a = \frac{1}{2} 
\estimPWithBayesJeffreys{\Celd}{}{\rvl^t}$ and $w_\Celd^b = \frac{1}{2} 
\estimPWithBayesJeffreys{\Celd}{}{\rvl^t}$, which correspond to those obtained after the initialization of Eq.~\ref{eq:weight-init} and the updates of Eq.~\ref{eq:weight-update}.
\end{lemma}
The correctness of our algorithm follows from Lem.~\ref{lem:delayed-equiv} and \cite[Thm. 4]{veness2012context}, since for $t\geq\tau_n(\Celd)$ the original update equations are used.
  

\subsection{Complexity}

The cost of processing $\rvl_n,\rvp_n$ is linear in the depth $D_n$ of the node split by the insertion of $\rvp_n$, since the algorithm updates the weights at each node in the path leading to this node. 
If $\probaid{\rvP}$ is absolutely continuous with respect to the Lebesgue measure, 
since the full-fledged k-d tree is monotone transformation invariant, we can assume without loss of generality 
that the marginal distributions of $\rvP$ are uniform in $[0,1]$ (see \cite[Sec. 20.1]{devroye1996probabilistic}) and thus its 
profile is equivalent to that of a random binary search tree under the random permutation model (see \cite[Sec. 2.3]{mahmoud1992evolution}).
Then, $D_n$ corresponds to the cost of an unsuccessful search and $\frac{D_n}{2\log n}\to 1$ in probability (see \cite[Sec. 2.4]{mahmoud1992evolution}). 
Therefore, the complexity of processing $\rvl_n,\rvp_n$ is $\bigO{\log n}$ in probability with respect to $\probaid{\seqZN}$.



\section{Experiments}
\label{sec:rpct-experiments}


\myparagraph{Software-hardware setup.} Python code and data used for the experiments are available at \url{https://github.com/alherit/kd-switch}.
Experiments were carried out on a machine running Debian 3.16,
equipped with two Intel(R) Xeon(R) E5-2667 v2 @ 3.30GHz processors and
62 GB of RAM.


\myparagraph{Boosting finite length performance with ensembling.}
When considering finite length performance, we can be unlucky and obtain a bad 
set of hierarchical partitions (i.e., with low discrimination power).
In order to boost the probability of finding good partitions, we
can use a Bayesian mixture of $\codecx{J}$ trees. 
Bayesian mixtures trivially maintain universality.
%

\myparagraph{Two sampling scenarii for labels.} 
In the first one, labels are sampled from a Bernoulli distribution
such that $\proba{\rvL=\labelX}=\theta_0$, where $\theta_0$ is a
known  parameter.  We then we sample from $\probaid{\rvP|\rvL}$. 
In this case, the root node distribution
$\estimPWithBayesJeffreys{}{}{\rvl_n|\seq{\rvl}{n-1}}$ is replaced by
$\proba{\rvLi{n}=\rvli{n}}=\theta_0^{\mathds{1}_{\{\rvli{n}=\labelX
    \}} } (1-\theta_0)^{\mathds{1}_{\{\rvli{n}=\labelY \}} }$, since
$\theta_0$ is known.
In the second one, observations come in random order and
$\proba{\rvL}$ is unknown.

\subsection{Normalized log loss (NLL) convergence} 
\label{sec:ll-convergence}

\myparagraph{Datasets.} 
We use the following datasets, detailed in Appendix \ref{sec:log-loss-datasets}:
{\bf (L-i)} A $2D$ dataset consists of two Gaussian Mixtures spanning
three different scales.
{\bf (L-ii)} A dataset in dimension $d=784$ composed of both real
MNIST digits, as well as digits generated by a Generative Adversarial
Network \cite{RadfordMC15} trained on the MNIST
dataset. 
{\bf (L-iii)} The Higgs dataset \cite{Lichman:2013}, the goal being to
distinguish the signature of processes producing Higgs bosons.
{\bf (L-iv)} The Breast Cancer Wisconsin (Diagnostic) Data Set
\cite{Lichman:2013}---dimension $d=30$.  

For cases (L-i,L-ii,L-iii), in order to feed the online predictors, we
apply the first sampling scenario for labels.
For case (L-iv), we apply the second one and, in
each trial, we take the pooled dataset in a random order to feed the
online predictors.

\myparagraph{Results.}  We focus on the cumulative normalized log loss
performance (NLL), and the trade-off with the computational
requirements---by limiting the running time to 30'.

We compare the performance of our online predictors $\estimPWithKDSi$
and $\estimPWithKDWi$ (see Rmk.~\ref{rmk:kdw-def}) with a number of
trees $\codecx{J}\in \{1,50\}$, against the following contenders.
The Bayesian mixture of knn-based sequential regressors proposed in \cite{lheritier2018sequential},
with a switch distribution using a horizon-free prior as $\estimPWithKDSi$. 
Practically, this predictor depends on a given set of functions of $n$ specifying the number of neighbors.
We use the same set specified in \cite{lheritier2018sequential}.
We also consider a Bayesian Mixture
  of Gaussian Processes Classifiers (gp) with RBF kernel width
  $\sigma\in\{2^{4i}\}_{i=-5...7}$. 
(Our implementation uses the scikit-learn
  GaussianProcessClassifier~\cite{scikit-learn}.  For each
  observation, we retrain the classifier using all past
  observations---a step requiring samples from the two populations.
  Thus, we predict with a uniform distribution (or $\proba{\rvL}$ when
  known) until at least one instance of each label has been observed.)
In the case (L-i), we also compare to the true conditional
probability, which can be easily derived since $\probaid{\rvP|\rvL}$ are
known.  Note that the normalized log loss of the true conditional
probability converges to the conditional entropy by the
Shannon–McMillan–Breiman theorem \cite{cover2006elements}.

Fig.~\ref{fig:log-loss} (Left and Middle) shows the NLL convergence with respect to the number of
	samples. Notice that due to the 30' running time budget, curves stop at different $n$.
	 Fig.~\ref{fig:log-loss} (Right) illustrates the computational complexity of each method.
For our predictors, the statistical efficiency increases with the
number of trees---at the expense of the computational burden.
Weighting performs better than switching for datasets (L-i, L-ii,
L-iv), and the other way around for (L-iii). 
knn takes some time to get the {\em right scale}, then converges fast in
most cases---with a plateau though on dataset L-ii though.  This
owes to the particular set of exponents used to define the mixture of
regressors~\cite{lheritier2018sequential}.
Also, knn is computationally demanding, in
particular when compared to our predictors with $\codecx{J}=1$.

\begin{figure}
\begin{center}
\includegraphics[width=.325\linewidth]{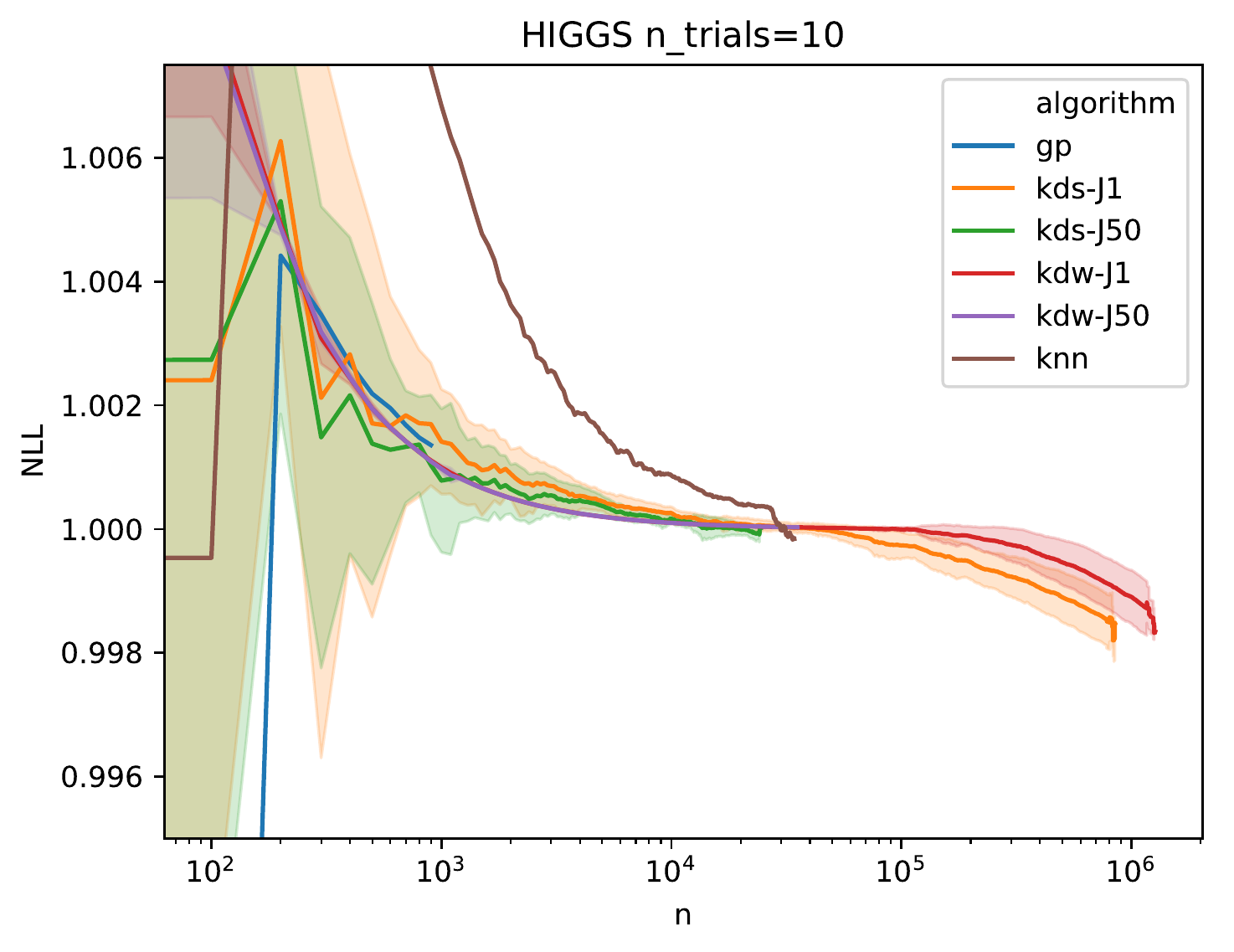} 
\includegraphics[width=.325\linewidth]{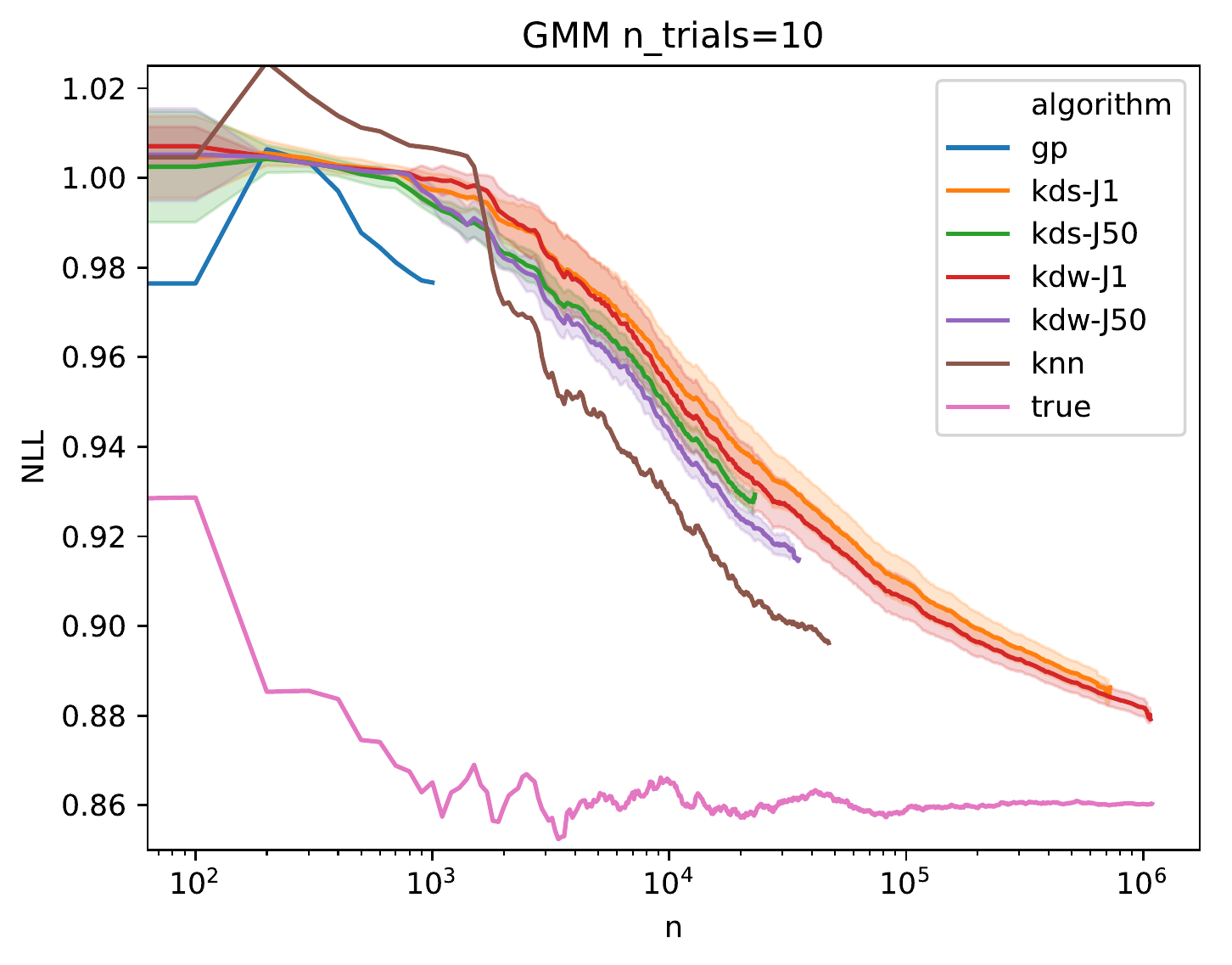} 
\includegraphics[width=.325\linewidth]{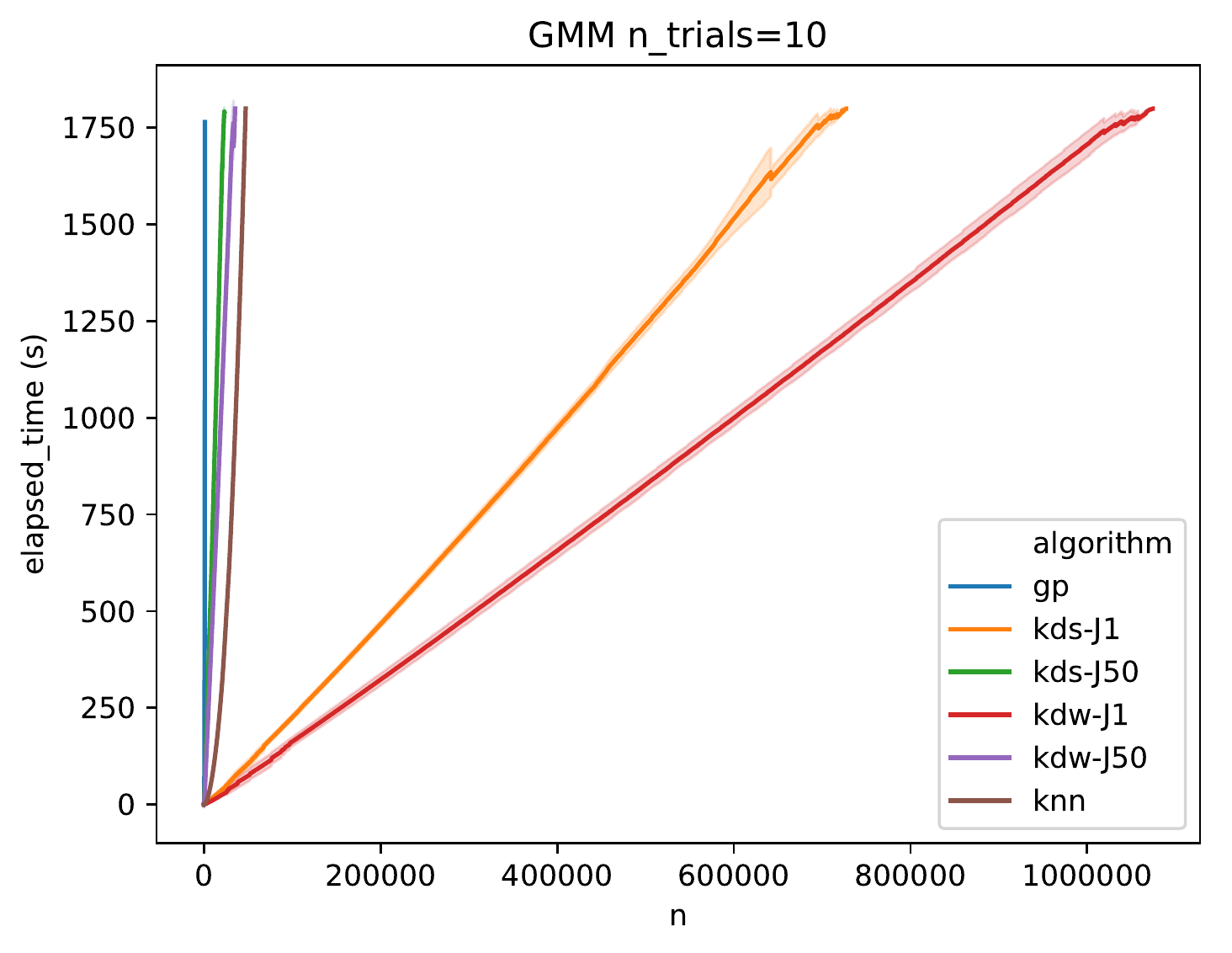}\\
\includegraphics[width=.325\linewidth]{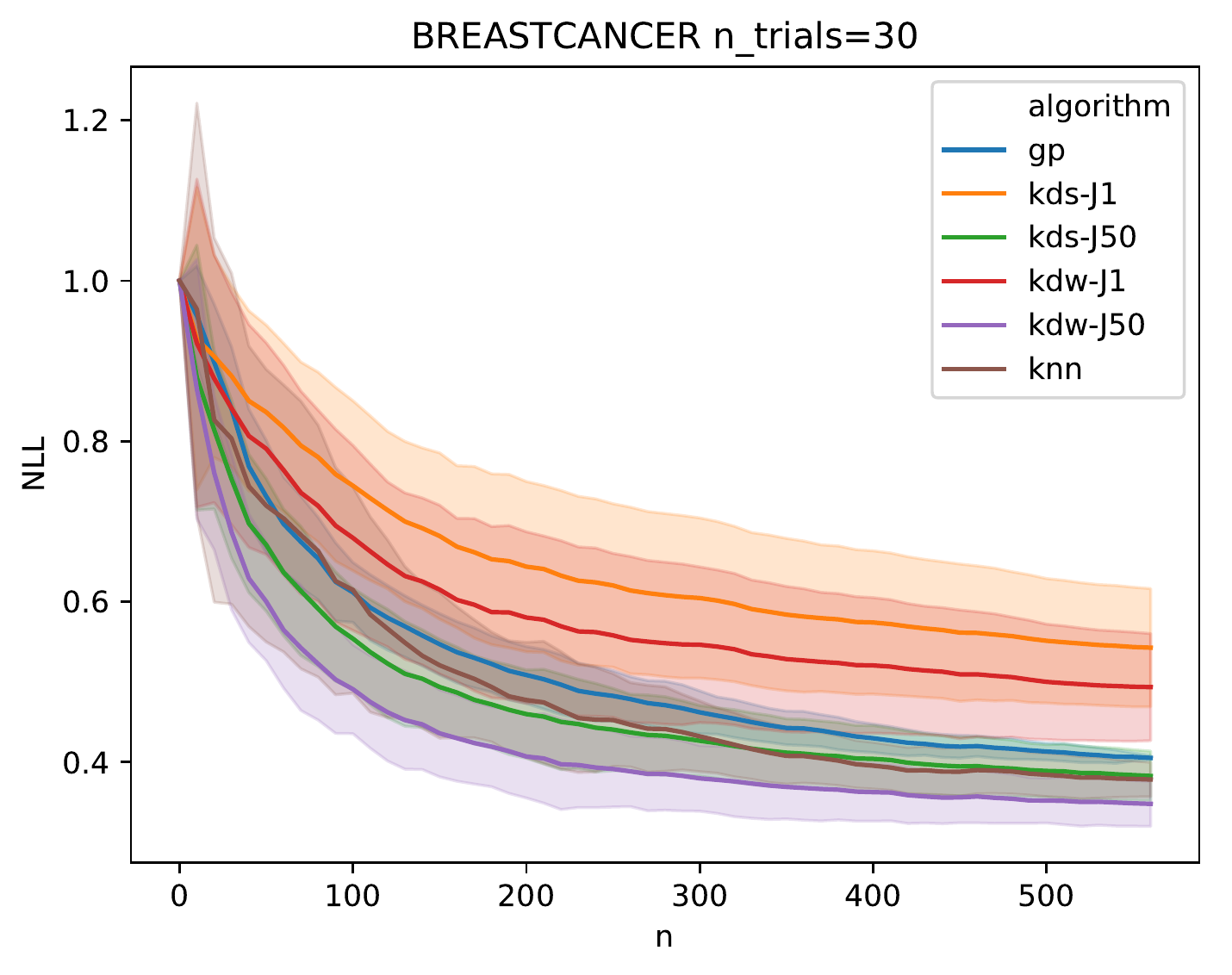}
\includegraphics[width=.325\linewidth]{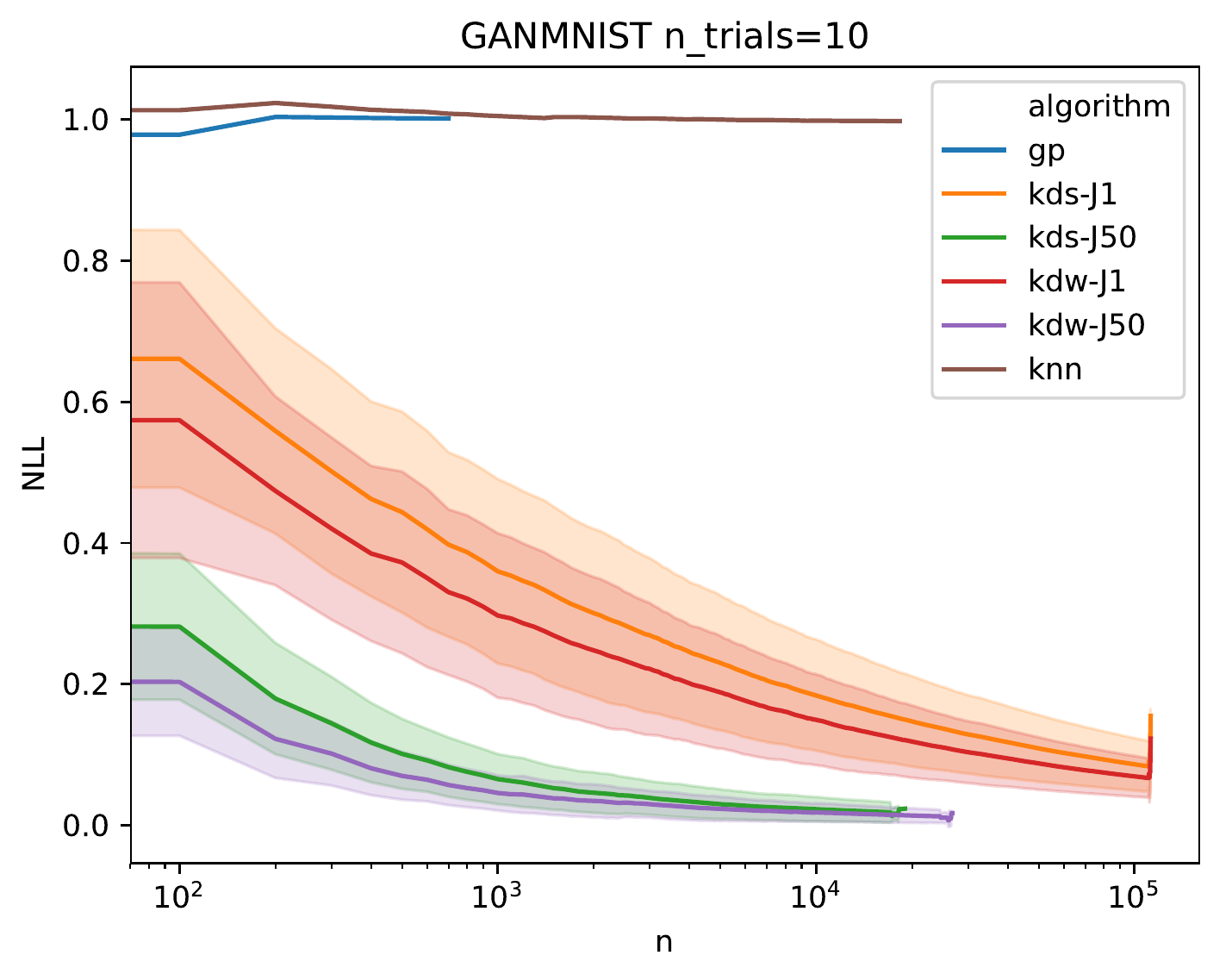} 
\includegraphics[width=.325\linewidth]{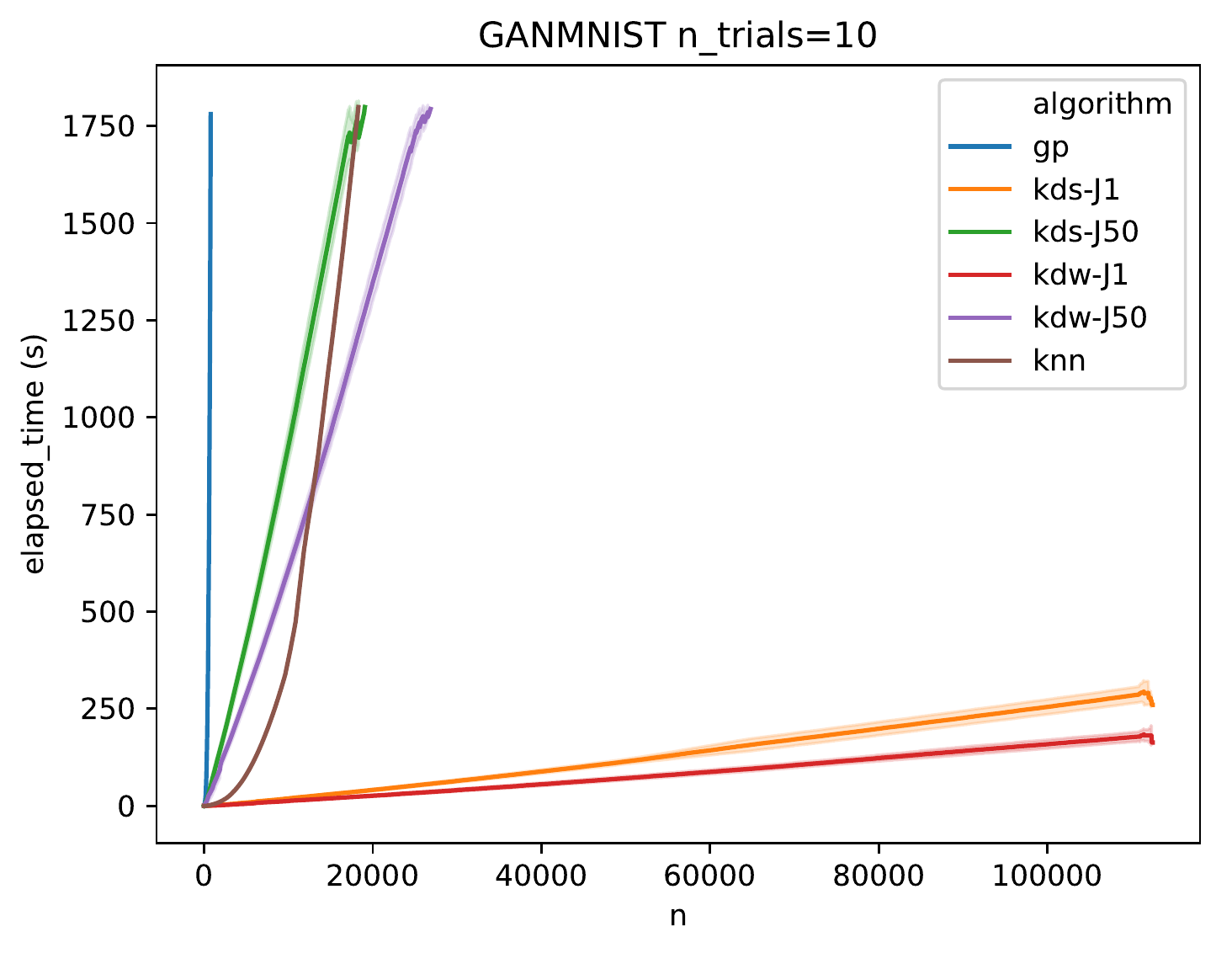}
\end{center}
\caption{
{\bf (Left and Middle)} 
Convergence of NLL as a function of $n$,  for a 30' calculation.
{\bf (Right)} 
Running time as a function of $n$.
Error bands represent the std dev.~w.r.t. the randomness
in the tree generation except for dataset (L-iv) where they represent the std dev.~w.r.t. the shuffling of the data.
}
\label{fig:log-loss}
\end{figure}


%
%

\subsection{Two-sample testing (TST)}
\label{sec:tst}

\myparagraph{Construction.} Given samples from two distributions,
whose corresponding random variables $X\in\Rd$ and $Y\in\Rd$ are
$\iid$, a nonparametric two-sample test tries to determine whether the
null hypothesis $\probaid{X} = \probaid{Y}$ holds or not (see, e.g., \cite[Section 6.9]{lehmann2005testing}).
Consistent sequential two-sample tests with optional stop (i.e. the
$p$-value is valid at any time $n$) can be built from a pointwise
universal online predictor $\arbDistQd$ \cite{lheritier2018sequential} by defining
$(L, Z)$ as:  $(0, X)$ with probability $\theta_0$, or  $(1, Y)$ with
probability $1-\theta_0$, where $\theta_0$ is a design parameter set to $1/2$ in the following experiments. The $p$-value is the likelihood ratio $\frac{\proba{\seqlN} } 
{ \arbDistQ{\seqlN|\seqzN}}$. Note this corresponds to the first sampling
scenario for labels.
The instantiation of this construction with $\estimPWithKDSi$ and $\codecx{J}=50$  is denoted \chronoswitchtst.

\myparagraph{Contenders.}
We compare \chronoswitchtst against \swsbmapfs from
\cite{lheritier2018sequential}, denoted \knnbmaswitch:
a sequential two-sample test obtained by instantiating the construction described above with the online knn 
predictor described in the previous section.
%
We also compare \chronoswitchtst against the kernel tests
from \cite{jitkrittum2016interpretable}:  \mefull, \megrid,
\scffull, \scfgrid, \mmdquad, \mmdlin, and the classical Hotelling's
\ttwo for differences in means under Gaussian assumptions.
 These tests depend on a kernel width $\sigma$ learned on a
trained set---the train-test paradigm---as opposed to \chronoswitchtst
which automatically detects the pertinent scales.
Contenders were launched with the hyperparameters specified in their
respective paper.  
For a fair comparison between sequential methods and those tests using
the train-test paradigm with $n_\text{test}$ used for testing, we use
a number of samples $n=4n_\text{test}$---detail in Appendix
\ref{sec:tst-datasets}.

\myparagraph{Datasets.}
\label{exp:1}
We use the four datasets from \cite[Table 1]{jitkrittum2016interpretable}: 
{\bf (T-i)} Same Gaussians in dimension $d=50$, to assess the
type I error;
{\bf (T-ii)}   Gaussian Mean Difference (GMD): normal distributions with a difference
in means along one direction, $d=100$;
{\bf (T-iii)} Gaussian Variance Difference (GVD): normal 
distributions with a difference in variance along one direction,  $d=50$;
{\bf (T-iv)}
Blobs (Mixture of Gaussian distributions centered on a
lattice)  \cite{gretton2012optimal}.
Datasets (T-ii, T-iii, T-iv) are meant to assess type II error.
To prevent k-d tree cuts to exploit the particular direction where the difference 
lies, these datasets undergo a random rotation (one per tree). See Appendix \ref{sec:tst-without-random-rot} for results without rotations.

\newcommand{\tableContent}{
	\centering
	\begin{tabular}{|l||l||l|l|l|}
		\hline
		Case  &  d & \scfgrid & \knnbmaswitch & \chronoswitchtst \\
		\hline
		GMD & ..& .. & .. & .. \\
		GVD & .. & .. & .. & ..\\
		SG & ..& .. & .. & .. \\
		BLOBS & ..& .. & .. & .. \\
		\hline
		
	\end{tabular}
	\caption{ {\bf CAPTION.} caption. }
	\label{tab:res-higgs} 
}

\begin{figure}
		\includegraphics[height=20mm]{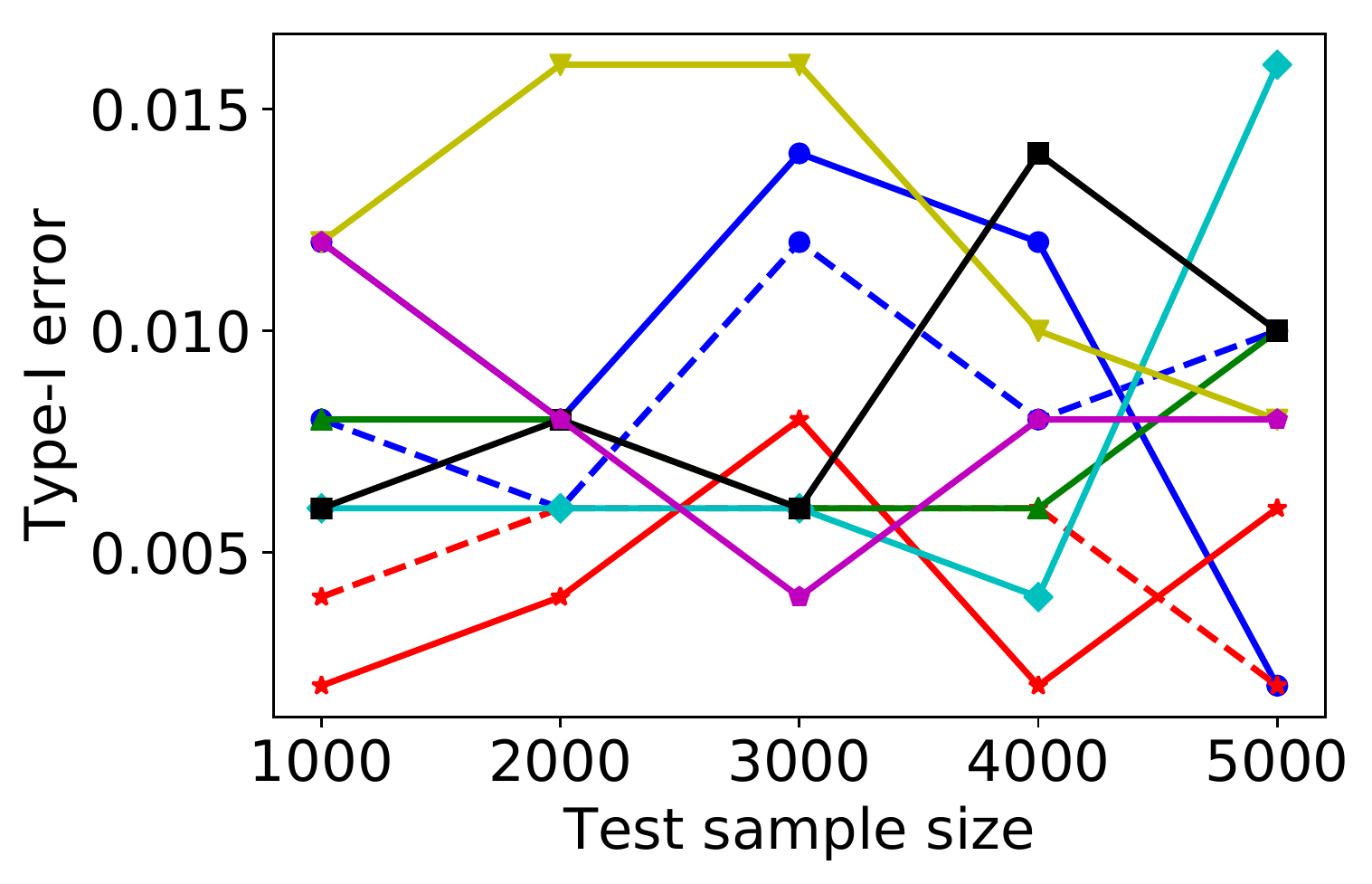} 
		\includegraphics[height=20mm]{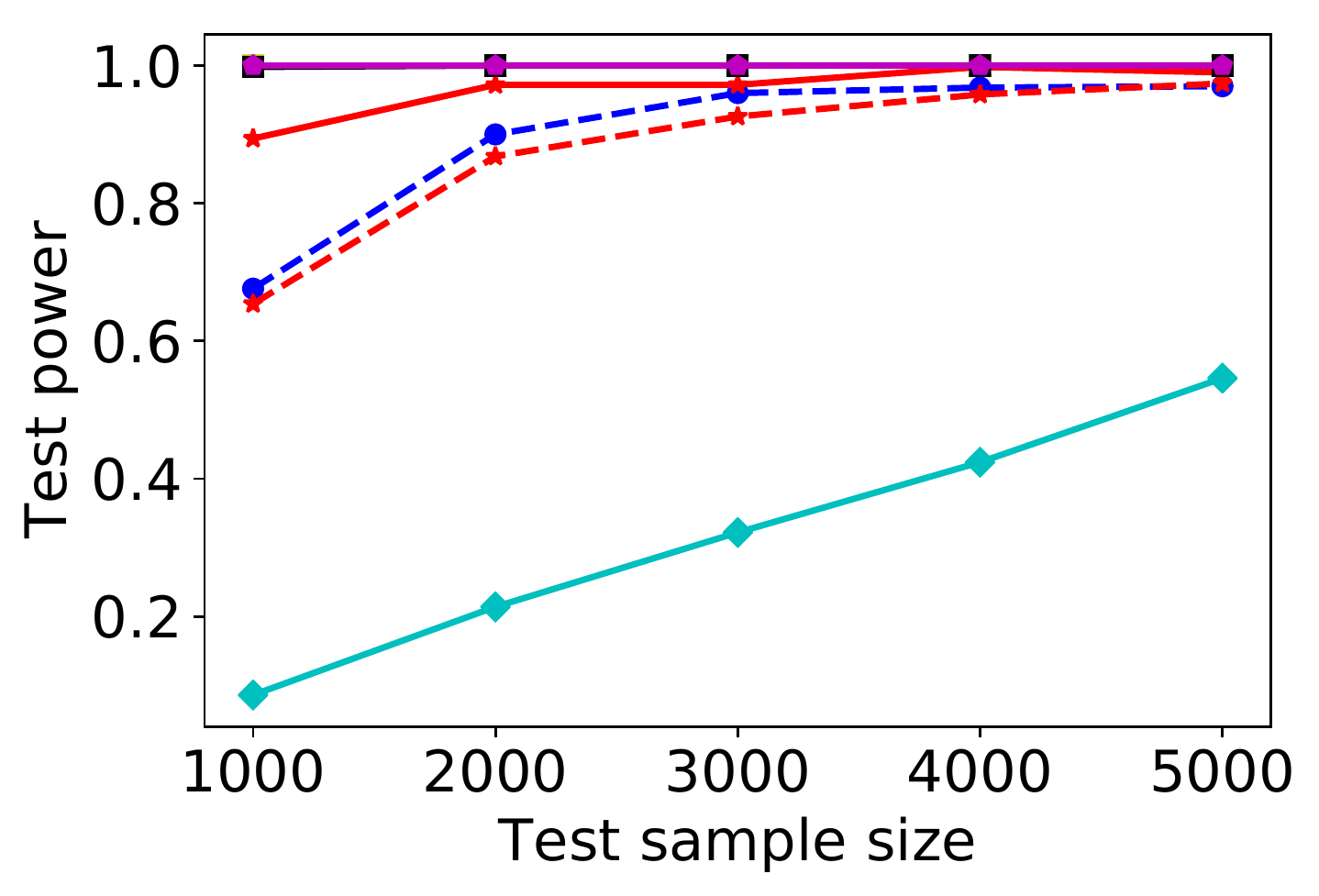} 
		\includegraphics[height=20mm]{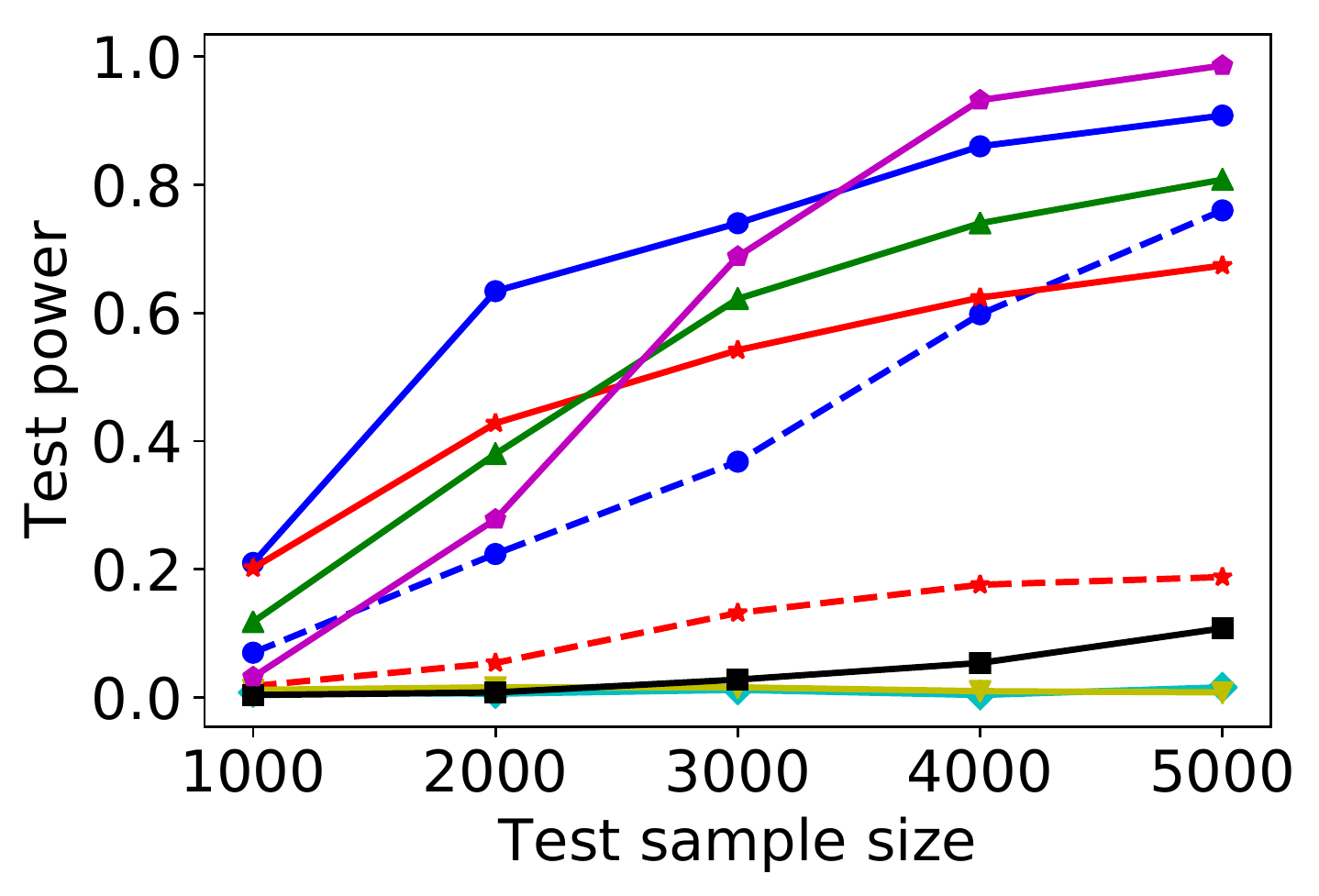} 
		\includegraphics[height=20mm]{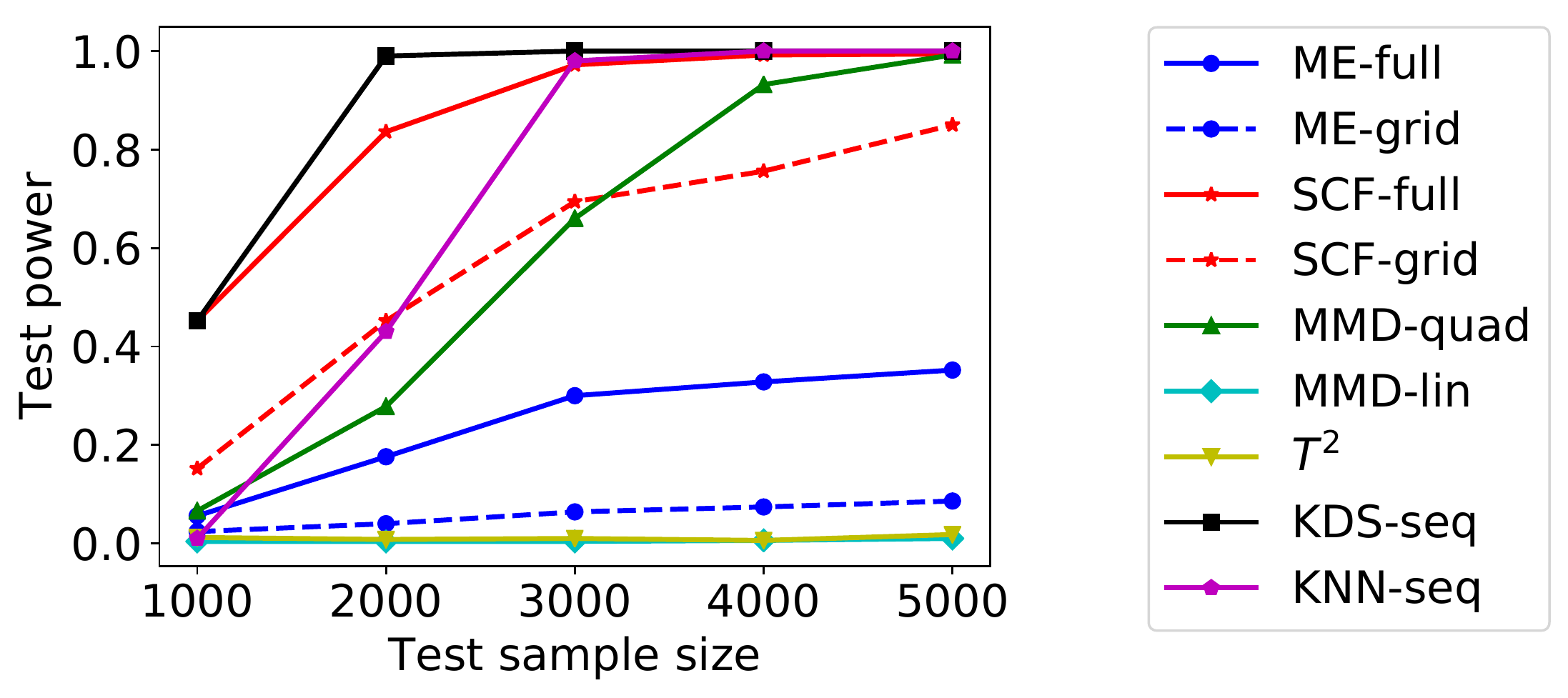}
		\begin{tabular}{cccc}
			~~~(a) SG. $d = 50$.~~~~~ & (b) GMD. $d = 100$. & (c) GVD. $d = 50$. & ~~(d) Blobs. $d = 2$.
		\end{tabular}
		
	\caption{{\bf Tests on randomly rotated Gaussian datasets from \cite{jitkrittum2016interpretable}.} The abscissa represents the test sample size $n_{\text{test}}$ for each of the two samples. 
Thus, for sequential methods, $n=4n_\text{test}$.}
	\label{fig:tests-gd} 
\end{figure} 

\myparagraph{Results.}  The significance level is set to
  $\alpha=.01$ in all the cases. The Type I error rate and the power ($1 - \text{Type II error rate}$) are computed over 500 trials. In the SG case (Fig.~\ref{fig:tests-gd}(a)), all the tests have a Type I error rate around the specified $\alpha$ as expected.
In the GMD and Blobs cases (Fig.~\ref{fig:tests-gd}(b,d)),
\chronoswitchtst matches or outperforms all the contenders.  
  On Blobs, \chronoswitchtst outperforms \knnbmaswitch
thanks to its automatic scale detection, even though the mixture
used by the latter allows it to handle the multiple scales.
For GVD (Fig.~\ref{fig:tests-gd}(c)), our results are weaker.
To see why, recall that GMD is generated by adding one unit to one
coordinate of the mean vector, while GVD is obtained by doubling the
variance along one direction. The span of the latter dataset
is larger, and upon rotating the data---see comment above---all
directions are impacted.  Given the high dimensionality, the
partitioning of k-d trees faces more difficulties to reduce the diameter 
of cells, which is key to convergence---see Corollary~\ref{cor:silva}.


\section{Outlook}

We foresee the following research directions.
A first open question is to characterize the situations where
switching should be preferred over weighting.
A second core question is to quantify the ability of our k-d tree based
construction to cope with multiple scales in the data.
A third one is the derivation of finite length bounds related to the
{\em complexity} of the underlying conditional distribution.
Finally,  accommodating data in a metric (non Euclidean)
space, using e.g.~metric trees, would widen the application spectrum
of the method.


\subsubsection*{Acknowledgments}
We would like to thank María Zuluaga, Eoin Thomas, Nicolas Bondoux and Rodrigo Acuña-Agost for insightful comments and, also, Wittawat Jitkrittum and Arthur Gretton for providing us the complete output of their experiments.

{\small
	\bibliographystyle{abbrv}
	\bibliography{mcs,abs}%

\begin{thebibliography}{10}

\bibitem{algoet1992universal}
P.~Algoet.
\newblock Universal schemes for prediction, gambling and portfolio selection.
\newblock {\em The Annals of Probability}, 20(2):901--941, 1992.

\bibitem{baldi2014searching}
P.~Baldi, P.~Sadowski, and D.~Whiteson.
\newblock Searching for exotic particles in high-energy physics with deep
  learning.
\newblock {\em Nature communications}, 5:4308, 2014.

\bibitem{Barron1998information}
A.~R. Barron.
\newblock Information-theoretic characterization of {B}ayes performance and the
  choice of priors in parametric and nonparametric problems.
\newblock In A.~D. J.M.~Bernardo, J.O.~Berger and A.~Smith, editors, {\em
  Bayesian Statistics 6}, pages 27--52. Oxford University Press, 1998.

\bibitem{begleiter2006superior}
R.~Begleiter and R.~El-Yaniv.
\newblock Superior guarantees for sequential prediction and lossless
  compression via alphabet decomposition.
\newblock {\em Journal of Machine Learning Research}, 7(Feb):379--411, 2006.

\bibitem{cai2005universal}
H.~Cai, S.~R. Kulkarni, and S.~Verd{\'u}.
\newblock A universal lossless compressor with side information based on
  context tree weighting.
\newblock In {\em Information Theory, 2005. ISIT 2005. Proceedings.
  International Symposium on}, pages 2340--2344. IEEE, 2005.

\bibitem{cesa2006prediction}
N.~Cesa-Bianchi and G.~Lugosi.
\newblock {\em Prediction, learning, and games}.
\newblock Cambridge university press, 2006.

\bibitem{cover2006elements}
T.~Cover and J.~Thomas.
\newblock {\em Elements of Information Theory}.
\newblock Wiley \& Sons, 2006.

\bibitem{devroye1996probabilistic}
L.~Devroye, L.~Gy{\"o}rfi, and G.~Lugosi.
\newblock {\em A probabilistic theory of pattern recognition}, volume~31.
\newblock Springer Verlag, 1996.

\bibitem{gretton2012optimal}
A.~Gretton, D.~Sejdinovic, H.~Strathmann, S.~Balakrishnan, M.~Pontil,
  K.~Fukumizu, and B.~Sriperumbudur.
\newblock Optimal kernel choice for large-scale two-sample tests.
\newblock In {\em Advances in Neural Information Processing Systems}, pages
  1205--1213, 2012.

\bibitem{grunwald2007minimum}
P.~D. Gr{\"u}nwald.
\newblock {\em The minimum description length principle}.
\newblock MIT press, 2007.

\bibitem{gyofi2005strategies}
L.~Gy{\"o}fi and G.~Lugosi.
\newblock Strategies for sequential prediction of stationary time series.
\newblock In {\em Modeling uncertainty}, pages 225--248. Springer, 2005.

\bibitem{gyorfi1999simple}
L.~Gyorfi, G.~Lugosi, and G.~Morvai.
\newblock A simple randomized algorithm for sequential prediction of ergodic
  time series.
\newblock {\em IEEE Transactions on Information Theory}, 45(7):2642--2650,
  1999.

\bibitem{hall1988stochastic}
P.~Hall and E.~Hannan.
\newblock On stochastic complexity and nonparametric density estimation.
\newblock {\em Biometrika}, 75(4):705--714, 1988.

\bibitem{jeffreys1946invariant}
H.~Jeffreys.
\newblock An invariant form for the prior probability in estimation problems.
\newblock {\em Proceedings of the Royal Society of London. Series A.
  Mathematical and Physical Sciences}, 186(1007):453--461, 1946.

\bibitem{jitkrittum2016interpretable}
W.~Jitkrittum, Z.~Szab{\'o}, K.~P. Chwialkowski, and A.~Gretton.
\newblock Interpretable distribution features with maximum testing power.
\newblock In {\em Advances in Neural Information Processing Systems}, pages
  181--189, 2016.

\bibitem{kozat2007universal}
S.~S. Kozat, A.~C. Singer, and G.~C. Zeitler.
\newblock Universal piecewise linear prediction via context trees.
\newblock {\em IEEE Transactions on Signal Processing}, 55(7):3730--3745, 2007.

\bibitem{trofimov1981performance}
R.~Krichevsky and V.~Trofimov.
\newblock The performance of universal encoding.
\newblock {\em Information Theory, IEEE Transactions on}, 27(2):199--207, 1981.

\bibitem{lehmann2005testing}
E.~L. Lehmann and J.~P. Romano.
\newblock {\em Testing statistical hypotheses}.
\newblock Springer Texts in Statistics. Springer, New York, third edition,
  2005.

\bibitem{lheritier2018sequential}
A.~Lh\'eritier and F.~Cazals.
\newblock A sequential non-parametric multivariate two-sample test.
\newblock {\em IEEE Transactions on Information Theory}, 64(5):3361--3370,
  2018.

\bibitem{Lichman:2013}
M.~Lichman.
\newblock {UCI} machine learning repository, 2013.

\bibitem{mahmoud1992evolution}
H.~Mahmoud.
\newblock {\em Evolution of random search trees}.
\newblock Wiley-Interscience, 1992.

\bibitem{merhav1998universal}
N.~Merhav and M.~Feder.
\newblock Universal prediction.
\newblock {\em Information Theory, IEEE Transactions on}, 44(6):2124--2147,
  1998.

\bibitem{scikit-learn}
F.~Pedregosa, G.~Varoquaux, A.~Gramfort, V.~Michel, B.~Thirion, O.~Grisel,
  M.~Blondel, P.~Prettenhofer, R.~Weiss, V.~Dubourg, J.~Vanderplas, A.~Passos,
  D.~Cournapeau, M.~Brucher, M.~Perrot, and E.~Duchesnay.
\newblock Scikit-learn: Machine learning in {P}ython.
\newblock {\em Journal of Machine Learning Research}, 12:2825--2830, 2011.

\bibitem{RadfordMC15}
A.~Radford, L.~Metz, and S.~Chintala.
\newblock Unsupervised representation learning with deep convolutional
  generative adversarial networks.
\newblock In {\em {ICLR}}, 2016.

\bibitem{rasmussen2005}
C.~E. Rasmussen and C.~K.~I. Williams.
\newblock {\em Gaussian Processes for Machine Learning (Adaptive Computation
  and Machine Learning)}.
\newblock The MIT Press, 2005.

\bibitem{rissanen1992density}
J.~Rissanen, T.~Speed, and B.~Yu.
\newblock Density estimation by stochastic complexity.
\newblock {\em Information Theory, IEEE Transactions on}, 38(2):315--323, 1992.

\bibitem{shtar1987universal}
Y.~Shtar'kov.
\newblock Universal sequential coding of single messages.
\newblock {\em Problemy Peredachi Informatsii}, 23(3):3--17, 1987.

\bibitem{silva2008optimal}
J.~Silva.
\newblock {\em On optimal signal representation for statistical learning and
  pattern recognition}.
\newblock PhD thesis, University of Southern California, 2008.

\bibitem{tziortziotis2014cover}
N.~Tziortziotis, C.~Dimitrakakis, and K.~Blekas.
\newblock Cover tree bayesian reinforcement learning.
\newblock {\em The Journal of Machine Learning Research}, 15(1):2313--2335,
  2014.

\bibitem{valiant1984theory}
L.~G. Valiant.
\newblock A theory of the learnable.
\newblock {\em Communications of the ACM}, 27(11):1134--1142, 1984.

\bibitem{erven2012catching}
T.~van Erven, P.~Gr{\"u}nwald, and S.~de~Rooij.
\newblock Catching up faster by switching sooner: a predictive approach to
  adaptive estimation with an application to the {AIC--BIC} dilemma.
\newblock {\em Journal of the Royal Statistical Society: Series B (Statistical
  Methodology)}, 74(3):361--417, 2012.

\bibitem{veness2017online}
J.~Veness, T.~Lattimore, A.~Bhoopchand, A.~Grabska-Barwinska, C.~Mattern, and
  P.~Toth.
\newblock Online learning with gated linear networks.
\newblock {\em arXiv preprint arXiv:1712.01897}, 2017.

\bibitem{veness2012context}
J.~Veness, K.~S. Ng, M.~Hutter, and M.~Bowling.
\newblock Context tree switching.
\newblock In {\em Data Compression Conference (DCC), 2012}, pages 327--336.
  IEEE, 2012.

\bibitem{willems1998context}
F.~M.~J. Willems.
\newblock The context-tree weighting method: Extensions.
\newblock {\em Information Theory, IEEE Transactions on}, 44(2):792--798, 1998.

\bibitem{willems1995context}
F.~M.~J. Willems, Y.~M. Shtarkov, and T.~J. Tjalkens.
\newblock The context-tree weighting method: Basic properties.
\newblock {\em Information Theory, IEEE Transactions on}, 41(3):653--664, 1995.

\bibitem{yu1992data}
B.~Yu and T.~Speed.
\newblock Data compression and histograms.
\newblock {\em Probability Theory and Related Fields}, 92(2):195--229, 1992.

\end{thebibliography}
}

\newpage
\onecolumn
\appendix
\section{Proofs}

\begin{proof}[Lemma \ref{lem:ctw-twice-universal}]
	We follow the lines of the proof of \cite[Thm. 3]{veness2012context}, except 
	that we need to take into account the chronology of cell creation.  
	For notational convenience we omit $\Pi$ in $\estimPWithCTSi{\cdot}{\Pi}$.
	
	
	Let $\tilde{A}$ denote the set of cells 
	corresponding to internal nodes in the tree structure leading to $A$.
	Given a cell $\Celd$ and $\seqlzN$, let us denote $\nbSamplesInCell{\Celd}\equiv\nbSamples{\Celdof{\seqzN}}$, $\underline{l}^{n_\Celd}\equiv \Celd(l^n)$ and $\underline{z}^{n_\Celd}\equiv \Celd(z^n)$. 
	Then, $\forall n\geq 1, \forall l^n\in\alphabet^{n}, \forall z^n\in\Omega^{n}$, by dropping the sum in Equation \ref{eq:cts-def} and choosing the model index sequences $bb\dots b$ for internal nodes and $aa\dots a$ for leaf nodes,  we have that for any cell $\Celd\in \tilde{A}$
	\begin{align}
	\estimPWithCTSu{\Celd}{\Celd(l^n)|\Celd(z^n)} &\geq 
	w_\Celd(b^{\nbSamplesInCell{\Celd}}) \estimPWithBayesJeffreys{\Celd}{}{\underline{l}^{	\tau_n(\Celd)-1 	}}
	\prod_{k=\tau_n(\Celd)}^{\nbSamplesInCell{\Celd}} \sum_{j=1}^{2} 
	\mathds{1}_{\{\underline{\rvp}_k\in\Celi{j}\}}
	\frac{\estimPWithCTSu{\Celi{j}}{\Celiof{j}{\seq{\underline{\rvl}}{k}}
			\vert\Celiof{j}{\seq{\underline{\rvp}}{k} }
	}}
	{\estimPWithCTSu{\Celi{j}}{ \Celiof{j}{\seq{\underline{\rvl}}{k}}\minuslast
			\vert\Celiof{j}{\seq{\underline{\rvp}}{k} }\minuslast
	}} \nonumber \\
	\label{eq:cts-bound-case-internal}
	&= w_\Celd(b^{\nbSamplesInCell{\Celd}}) \estimPWithBayesJeffreys{\Celd}{}{\underline{l}^{\tau_n(\Celd)-1 	}}
	\prod_{j=1}^{2}
	\frac{\estimPWithCTSu{\Celi{j}}{\Celiof{j}{\seq{\underline{\rvl}}{\nbSamplesInCell{\Celd}}}
			\vert\Celiof{j}{\seq{\underline{\rvp}}{\nbSamplesInCell{\Celd}} }
	}}
	{\estimPWithCTSu{\Celi{j}}{ \Celiof{j}{\seq{\underline{\rvl}}{\tau_n(\Celd)}}\minuslast
			\vert\Celiof{j}{\seq{\underline{\rvp}}{\tau_n(\Celd)} }\minuslast
	}}
	\end{align}
	and for any cell $\Celd\in A$
	\begin{equation}
	\label{eq:cts-bound-case-leaf}
	\estimPWithCTSu{\Celd}{\Celd(l^n)|\Celd(z^n)} \geq 
	w_\Celd(a^{\nbSamplesInCell{\Celd}}) \estimPWithBayesJeffreys{\Celd}{}{\Celd(l^n)}
	.
	\end{equation}

	Then, by repeatedly applying Eq.~\ref{eq:cts-bound-case-internal} at internal nodes and Eq.~\ref{eq:cts-bound-case-leaf} at leaf nodes, we obtain
	\begin{equation}
	\estimPWithCTSu{\Omega}{l^n|z^n} \geq 
	\left( \prod_{\Celd\in A} w_\Celd(a^{\nbSamplesInCell{\Celd}}) \right) 
	\left( \prod_{\Celd\in \tilde{A}} w_\Celd(b^{\nbSamplesInCell{\Celd}}) \right)  
	\left(\prod_{\Celd\in A} \estimPWithBayesJeffreys{\Celd}{}{\Celdof{l^n}} \right) \kappa 
	\end{equation}
	where $\kappa$ groups all the terms from Eq.~\ref{eq:cts-bound-case-internal} that do not depend on $n$. We have
	\begin{equation}
	w_\Celd(a^{\nbSamplesInCell{\Celd}}) = \frac{1}{2} 
	\prod_{t=2}^{\nbSamplesInCell{\Celd}} \frac{t-1}{t} = \frac{1}{2\nbSamplesInCell{\Celd}}  \geq \frac{1}{2n}
	\end{equation}
	where the factor $\frac{t-1}{t}$ for $t\geq2$ comes from the prior probability of not switching.  
	Analogously, $w_\Celd(b^{\nbSamplesInCell{\Celd}}) \geq \frac{1}{2n}$. (Note that, in a Context Tree Weighting scheme (see Remark \ref{rmk:kdw-def}), the prior probability of not switching is 1 and thus, the lower bound for $w_\Celd(a^{\nbSamplesInCell{\Celd}})$ and $w_\Celd(b^{\nbSamplesInCell{\Celd}})$ becomes $1/2$.)
	Then, 
	\begin{align}
	\estimPWithCTSu{\Omega}{l^n|z^n} \geq \kappa 
	(2n)^{-\Gamma_A} 
	\prod_{\Celd\in A}  \estimPWithBayesJeffreys{\Celd}{}{\Celdof{l^n}} 
	.
	\end{align}

	The claimed inequality follows since, from \cite[Eq. 23]{willems1995context} (see e.g. \cite[Eq. 17]{begleiter2006superior} for $\size{\alphabet}>2$), 
	\begin{equation}
	-\log \prod_{\Celd\in A} \estimPWithBayesJeffreys{\Celd}{}{\Celdof{l^n}} \leq
	|A|\zeta\left(\frac{n}{|A|}\right) -\log P_{\thetaVector{A}}(\seqlcondzN)
	.
	\end{equation} 
	The limit follows from the Shannon–McMillan–Breiman theorem (see, e.g., \cite{cover2006elements}).
\end{proof}

\begin{proof}[Corollary \ref{cor:silva}]
	The conditional entropy can be written as  
	\begin{align}
	\centri{\rvL}{\rvP}  &= \entr{\rvL} - \mutinfo{\rvL}{\rvP} \\
	&= \entr{\rvL} - \expeci{\probaid{\rvL}}{\divKL{\probaid{\rvP|\rvL}}{\probaid{\rvP}}}
	\end{align}
	where $\divKL{\cdot}{\cdot}$ denotes the Kullback-Leibler divergence.
	Since $\probaid{\rvP|\rvL}$ are absolutely continuous with respect to 
	$\probaid{\rvP}$ and these measures are 
	absolutely 
	continuous with respect to the Lebesgue measure, the claim follows from
	\cite[Thm. 4.2]{silva2008optimal}, which guarantees 
	\begin{equation*}
	\divKL{\probaid{\pi_n(\rvP|\seqZN)|\rvL}}{\probaid{\pi_n(\rvP|\seqZN)}}  \toas \divKL{\probaid{\rvP|\rvL}}{\probaid{\rvP}} 
	.\qedhere
	\end{equation*}
\end{proof}

\begin{proof}[Lemma \ref{lem:k-d-tree}]
	By Markov's inequality, it is sufficient to show that 
	\begin{equation}
	\expeci{\probaid{\rvP}}{\diameter(\pi_n(\rvP|\seqZN)) } \town 0 \as.
	\end{equation}
	As in \cite[Sec. 20.1]{devroye1996probabilistic}, since the k-d tree is 
	monotone transformation invariant, we can assume without loss of generality 
	that $\rvP\in[0,1]^d$.
	In the proof of \cite[Thm. 20.3]{devroye1996probabilistic}, it is shown 
	that 
	for any $\epsilon>0$ and any $x\in\Rd$
	\begin{equation}
	\label{eq:diam-cond-devroye}
	\diameter(\pi_n(x|\seqZN)) \leq 2 \epsilon \sqrt{d} 
	\end{equation}
	if some specific event $E(x,\seqZN,\epsilon)$ holds.
	(Devroye presents the proof for the case $d=2$, leaving the  straightforward
adaptation for $d>2$ to the reader. 
        The proof considers a fixed point $x$ and, upon inserting $k$
        points into the k-d tree, the maximum distance from
        $x$ to the $2^d$ faces  of its containing hyper-rectangle. The event
        $E(x,\seqZN,\epsilon)$ stipulates that this maximum distance
        is bounded by $\epsilon$ upon inserting a number of points
        which may be $k=n^{1/3}$ or $k=n^{2/3}$ or $k=n$.)
	Then, it is shown that one can construct a set $B\subset\Rd$ such that
	$\probai{\rvP}{\rvP\in B}=1$ and for all $x\in B$, and sufficiently
	small $\epsilon>0$, $\probai{\seqZN}{E(x,\seqZN,\epsilon)}\town 1$.
	\smallskip
	
	Then, for $\epsilon>0$ sufficiently small, by total probability, we have
	\begin{align}
	&\probai{\seqZN,\rvP}{E(\rvP,\seqZN,\epsilon)}  \\
	&=\probai{\seqZN,\rvP}{E(\rvP,\seqZN,\epsilon)|\rvP\in
		B}\probai{\rvP}{\rvP\in B} + 
	\probai{\seqZN,\rvP}{E(\rvP,\seqZN,\epsilon)|\rvP\notin
		B}\probai{\rvP}{\rvP\notin B} \\
	&=\probai{\seqZN,\rvP}{E(\rvP,\seqZN,\epsilon)|\rvP\in B}.
	\end{align}
	Since the desired property is defined w.r.t.~$\probaid{\rvP}$ while
	\refeq{eq:diam-cond-devroye} involves the sequence $\seqZN$, we apply
	the law of total expectation with the two events
	$E(\rvP,\seqZN,\epsilon)$ and $\neg E(\rvP,\seqZN,\epsilon)$:
	\begin{align}
	\begin{split}
	&\limn \expeci{\probaid{\rvP}}{\diameter(\pi_n(\rvP|\seqZN))}
	=\\ 
	&\limn 
	\condexpeci{\probaid{\rvP}}{\diameter(\pi_n(\rvP|\seqZN))}{E(\rvP,\seqZN,\epsilon)}\probai{\seqZN,\rvP}{E(\rvP,\seqZN,\epsilon)|\rvP\in B}
	+ \\
	&\condexpeci{\probaid{\rvP}}{\diameter(\pi_n(\rvP|\seqZN))}{\neg E(\rvP,\seqZN,\epsilon)}\probai{\seqZN,\rvP}{\neg E(\rvP,\seqZN,\epsilon)|\rvP\in	B}=
	\end{split}
	\\
	\begin{split}
	&\limn 
	\expeci{\probaid{\rvP}}{\diameter(\pi_n(\rvP|\seqZN))|E(\rvP,\seqZN,\epsilon)}\cdot
	\limn \probai{\seqZN,\rvP}{E(\rvP,\seqZN,\epsilon)|\rvP\in B}
	+ \\
	&\limn \expeci{\probaid{\rvP}}{\diameter(\pi_n(\rvP|\seqZN))|\neg 
		E(\rvP,\seqZN,\epsilon)}\cdot \limn \probai{\seqZN,\rvP}{\neg 
		E(\rvP,\seqZN,\epsilon)|\rvP\in B}=
	\end{split}\\
	&\limn 
	\expeci{\probaid{\rvP}}{\diameter(\pi_n(\rvP|\seqZN))|E(\rvP,\seqZN,\epsilon)}
	\end{align}
	where the last equality stems from $\diameter(\pi_n(\rvP|\seqZN)) < \infty$ 
	and $ \limn \probai{\seqZN,\rvP}{E(\rvP,\seqZN,\epsilon)|\rvP\in B}=1$, for sufficiently small $\epsilon>0$.
	The random variable $\diameter(\pi_n(\rvP|\seqZN)$ is bounded since $\rvP\in[0,1]^d$.
	Therefore, by Lebesgue's dominated convergence theorem:
	\begin{align}
	\limn 
	\condexpeci{\probaid{\rvP}}{\diameter(\pi_n(\rvP|\seqZN))}{E(\rvP,\seqZN,\epsilon)}
	= \condexpeci{\probaid{\rvP}}{\limn  
		\diameter(\pi_n(\rvP|\seqZN))}{E(\rvP,\seqZN,\epsilon)}
	= 0
	\end{align}
	since $\diameter(\pi_n(\rvP|\seqZN)) \leq  2 \epsilon \sqrt{d}$ if  
	$E(\rvP,\seqZN,\epsilon)$ holds, for any $\epsilon>0$.
	Therefore,
	\begin{align*}
	&\probai{\seqZN}{\limn 
		\expeci{\probaid{\rvP}}{\diameter(\pi_n(\rvP|\seqZN))}=0}=1 
	.\qedhere
	\end{align*}
\end{proof}

\begin{proof}[Proof of Theorem \ref{thm:chrono-universal}] 
By Lemma \ref{lem:k-d-tree}, for any 
$\eps>0$, there exists $n(\eps)$ such that $\forall n\geq n(\eps)$
\begin{equation}
\centri{\rvL}{\pi_n(\rvP|\seqZN)} < \centri{\rvL}{\rvP} + \eps \as.
\end{equation}
Then, by Lemma \ref{lem:ctw-twice-universal} 
\begin{equation}
- \limn \frac{1}{n}\log \estimPWithKDSi{\seqLcondZN} 
\leq \centri{\rvL}{\pi_{n(\eps)}\left(\rvP|\seq{\rvP}{n(\eps)}\right)} \as .
\end{equation}
The claim follows since $\eps$ can be arbitrary small.  
\end{proof}

\begin{proof}[Lemma \ref{lem:delayed-equiv}]
	We use induction on $t$ and we denote $w_{\Celd,t}^a$ the value of $w_\Celd^a$ at the end of the updates after observing $\rvl^t$ in $\Celd$. For $t=0$, it trivially holds since the observed sequence is empty. 
	For $1\leq t < \tau_n(\Celd)$, by Eq.~\ref{eq:final-estimate} and \ref{eq:weight-update2}
	\begin{equation}
	\begin{split}
	&w_{\Celd,t}^a = \alpha^\Celd_{t+1} \estimPWithBayesJeffreys{\Celd}{}{\rvl^t} +
	\beta^\Celd_{t+1} w_{\Celd,t-1}^a 
	\estimPWithBayesJeffreys{\Celd}{}{\rvl_t|\seq{\rvl}{t-1}} \\
	&w_{\Celd,t}^b = \alpha^\Celd_{t+1} \estimPWithBayesJeffreys{\Celd}{}{\rvl^t} +
	\beta^\Celd_{t+1} w_{\Celd,t-1}^b 
	\estimPWithBayesJeffreys{\Celd}{}{\rvl_t|\seq{\rvl}{t-1}}
	\end{split}
	.
	\end{equation}
	Using the inductive hypothesis, we get
	\begin{equation}
	\begin{split}
	&w_{\Celd,t}^a = \alpha^\Celd_{t+1} \estimPWithBayesJeffreys{\Celd}{}{\rvl^t} +
	\beta^\Celd_{t+1}  \frac{1}{2} \estimPWithBayesJeffreys{\Celd}{}{\rvl^{t}} = \frac{1}{2} \estimPWithBayesJeffreys{\Celd}{}{\rvl^t} \\
	&w_{\Celd,t}^b = \alpha^\Celd_{t+1} \estimPWithBayesJeffreys{\Celd}{}{\rvl^t} +
	\beta^\Celd_{t+1} \frac{1}{2} \estimPWithBayesJeffreys{\Celd}{}{\rvl^{t}}  = \frac{1}{2} \estimPWithBayesJeffreys{\Celd}{}{\rvl^t} 
	\end{split}
	.
	\end{equation}
\end{proof}

\section{Experiments}

\subsection{Datasets for the normalized log loss convergence analysis}
\label{sec:log-loss-datasets}

\paragraph{Multiscale Gaussian Mixture dataset.}
The Gaussian mixture dataset from Section \ref{sec:ll-convergence} is built 
by generating two Gaussian Mixture models, one for each label.
The means of the Gaussians are uniformly drawn from $[0,1]^2$ and are
the same for both mixtures. The weights are randomly drawn from a
Dirichlet distribution with parameters $(1,\dots,1)$ and are the same
for both the mixtures. The covariance matrices are randomly drawn from
an inverse Wishart distribution with $d+2$ degrees of freedom and a
scale parameter $I$ for a third of the components, $.01I$ for the
second third and $.0001I$ for the last one.
See Fig. \ref{fig:gauss-mixture-exple} for an illustration.

\begin{figure*}
\centering
\includegraphics[width=.6\linewidth]{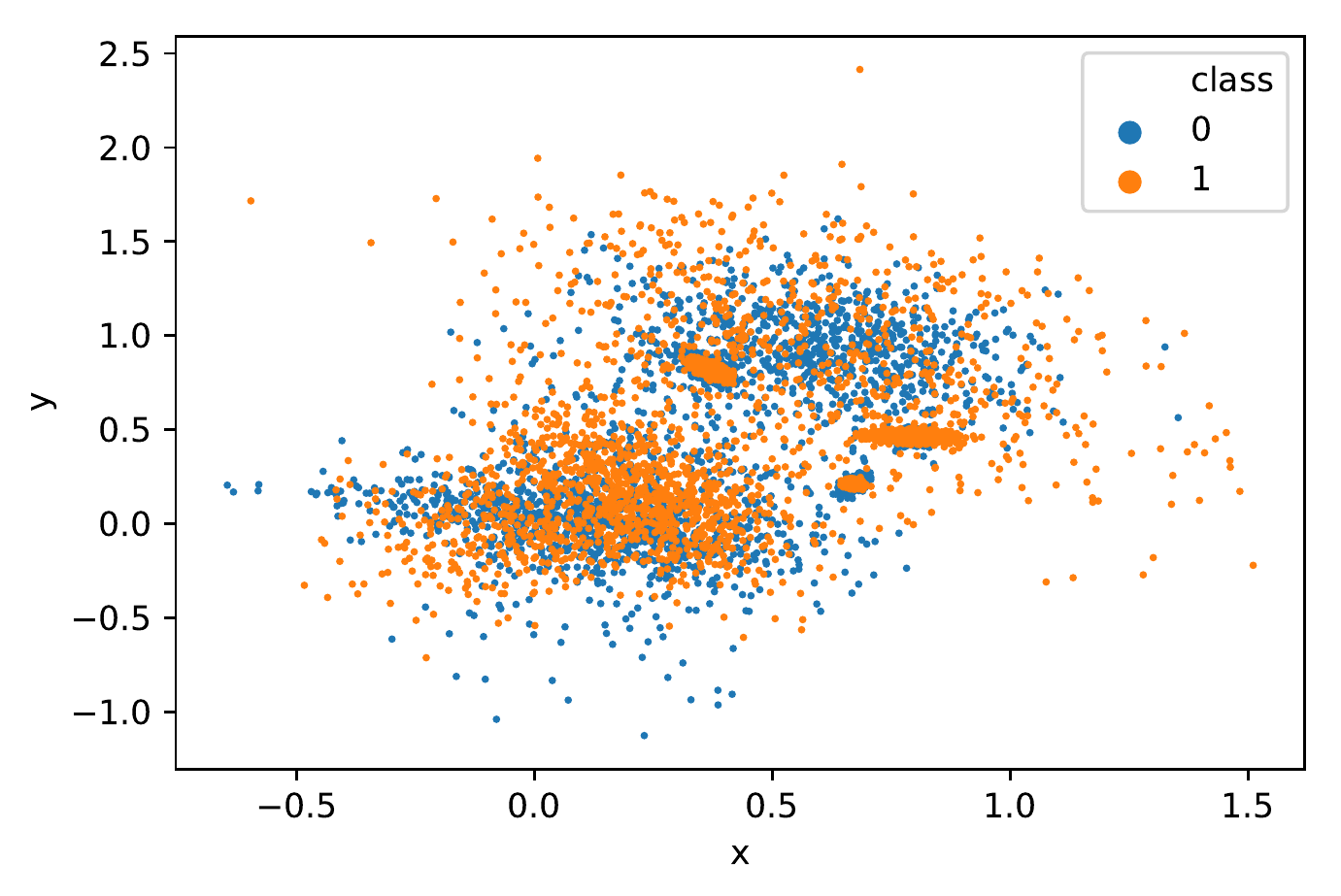}
\caption{{\bf Multiscale Gaussian Mixture dataset.}
Union of two sample sets, each from a  random mixture model.}
\label{fig:gauss-mixture-exple}
\end{figure*}

\paragraph{GAN dataset.}
We used the pretrained Deep Convolutional Generative Adversarial
Network available at
\url{https://github.com/csinva/pytorch_gan_pretrained}.  We generate
as many samples as real ones (60000).  We consider the real samples as
coming from $P_{Z|L=0}$ and the synthetic ones from $P_{Z|L=1}$. See
Fig. \ref{fig:gan} for an illustration.

\begin{figure*}
	\centering
	\includegraphics[width=.4\linewidth]{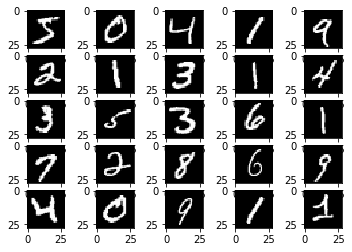}\hfill
	\includegraphics[width=.4\linewidth]{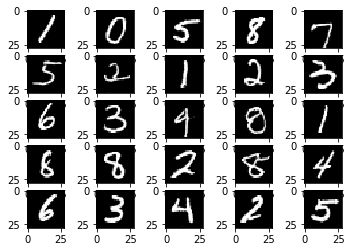}
	\caption{{\bf (Left.)} Examples of real images. {\bf (Right.)} Examples of GAN generated synthetic images.}
	\label{fig:gan}
\end{figure*}

\paragraph{HIGGS dataset.}
To distinguish the signature of processes producing Higgs bosons from
background processes which do not, we use the four low-level
features (azimuthal angular momenta $\phi$ for four particle jets)
which are known to carry very little discriminating information 
\cite{Lichman:2013,baldi2014searching}.

\subsection{Datasets and sampling for the Two-sample-test experiments}
\label{sec:tst-datasets}

\paragraph{Datasets.}
We use the four datasets from \cite[Table 1]{jitkrittum2016interpretable}: 
Same Gaussian (SG; $\probaid{X}$ and $\probaid{Y}$  are identical normal
distributions; $d=50$);
Gaussian Mean Difference (GMD; $\probaid{X}$ and $\probaid{Y}$ are normal
distributions with a difference in means along one direction; $d=100$);
Gaussian Variance Difference (GVD; $\probaid{X}$ and $\probaid{Y}$  are normal
distributions with a difference in variance along one direction; $d=50$);
Blobs (Mixture of Gaussian distributions centered on a
lattice)---a challenging case since differences occur at a much
smaller length scale compared to the global scale
\cite{gretton2012optimal}.
To prevent k-d tree cuts to exploit the particular direction where the difference 
lies, such datasets undergo a random rotation (one per tree).

\myparagraph{Sampling.}  For a fair comparison against
tests using the train-test paradigm,
sequential two-sample tests use a sample size equal to the sum of the training and
test set sizes used by the contenders. When we compare to these tests,
samples are obtained by the same sampling mechanism and with the same
random seed, using the code provided at
\url{https://github.com/wittawatj/interpretable-test}. Sequential
tests (i.e.~\chronoswitchtst and \knnbmaswitch) consume these samples
following the first sampling scenario specified in Section
\ref{sec:rpct-experiments}, with $\theta_0=.5$---labels are balanced.

\subsection{Two-sample-test experiments without random rotations}
\label{sec:tst-without-random-rot}

Figure \ref{fig:tests-gd-norot} shows the results on the original datasets without undergoing random rotations. We observe that k-d tree cuts are able to quickly detect the particular directions where the difference lies making the power significantly higher than for the randomly rotated case.

\begin{figure}
	\includegraphics[height=20mm]{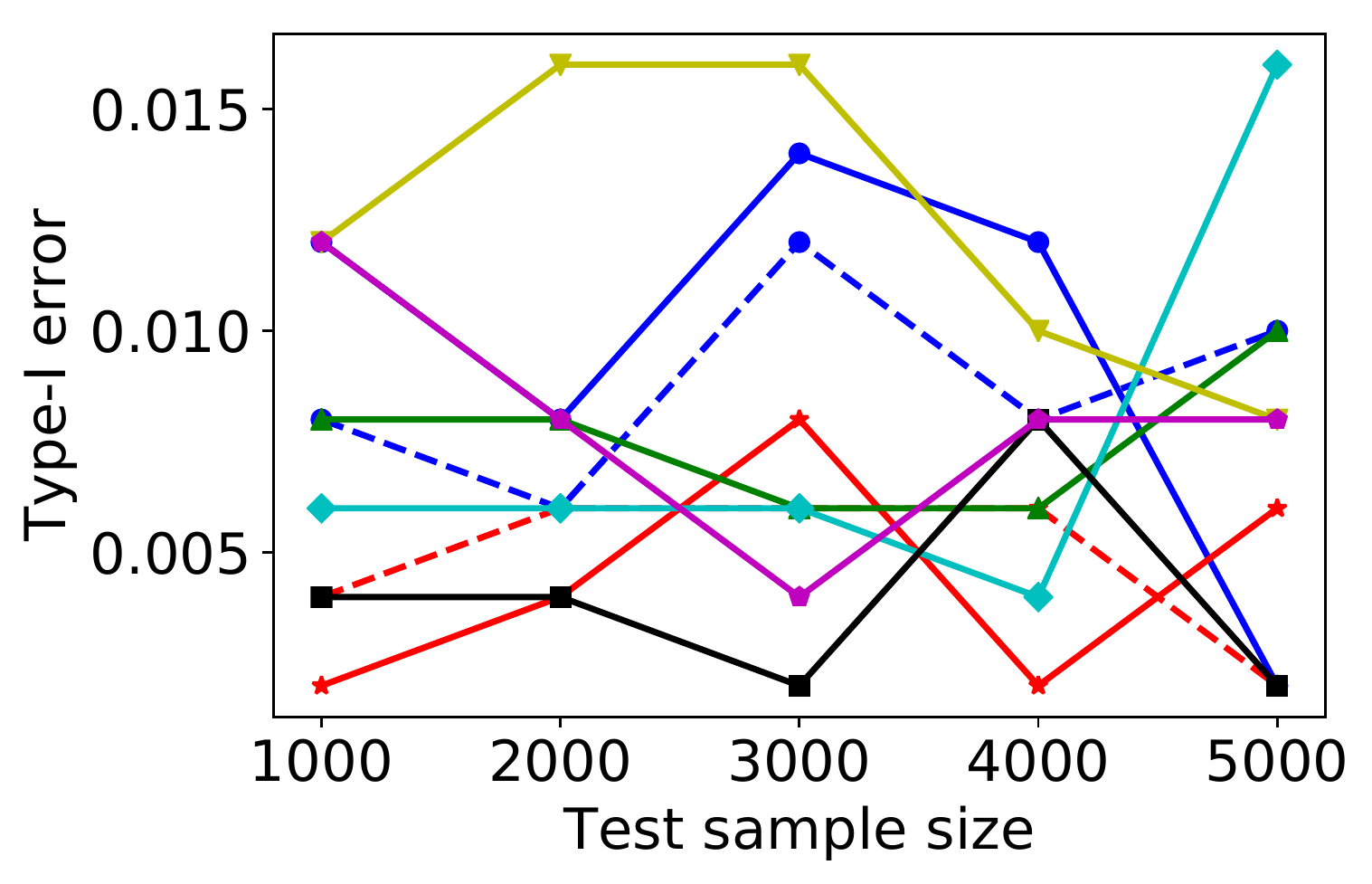} 
	\includegraphics[height=20mm]{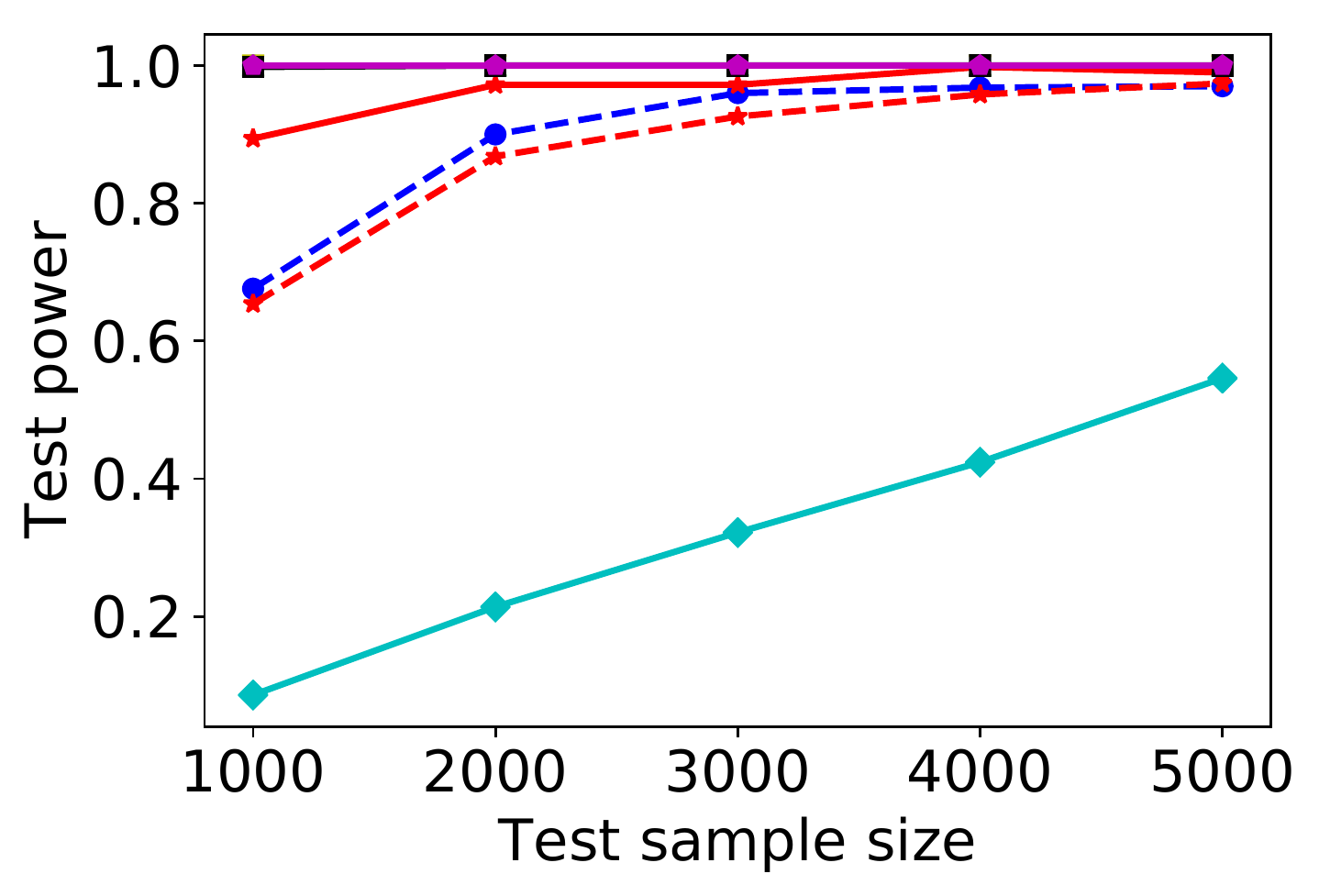} 
	\includegraphics[height=20mm]{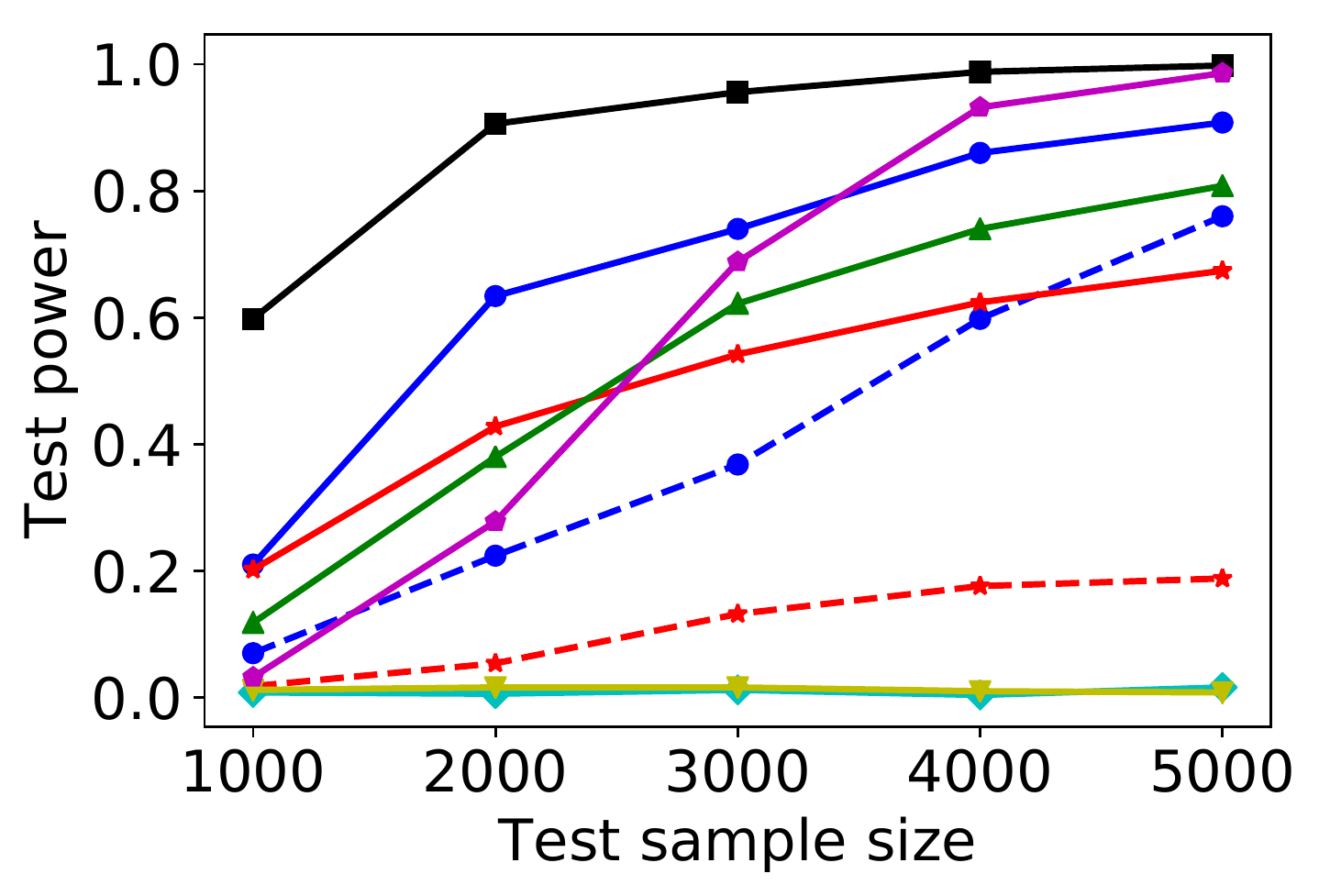} 
	\includegraphics[height=20mm]{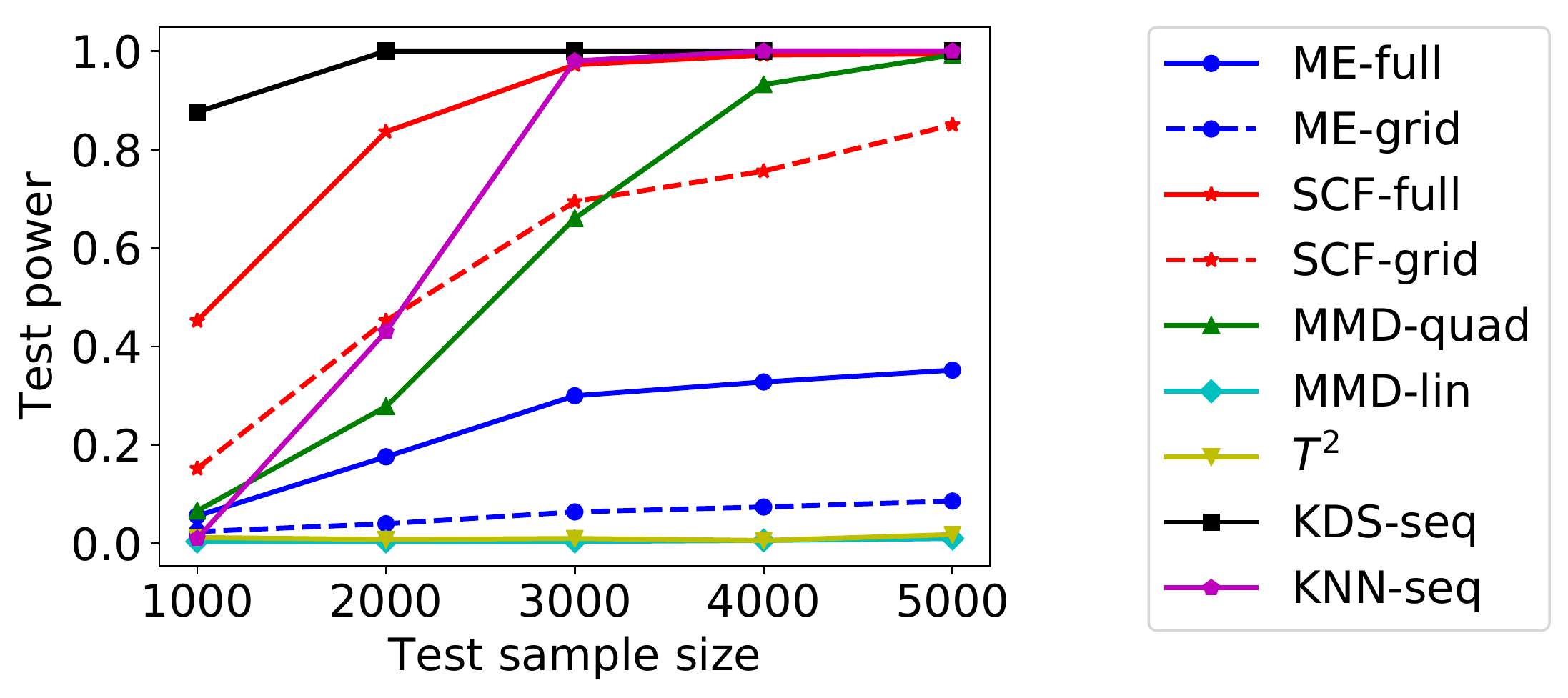}
	\begin{tabular}{cccc}
		~~~(a) SG. $d = 50$.~~~~~ & (b) GMD. $d = 100$. & (c) GVD. $d = 50$. & ~~(d) Blobs. $d = 2$.
	\end{tabular}
	
	\caption{{\bf Tests on non-rotated Gaussian datasets, as specified in \cite{jitkrittum2016interpretable}.} The abscissa represents the test sample size $n_{\text{test}}$ for each of the two samples. 
		Thus, for sequential methods, $n=4n_\text{test}$.}
	\label{fig:tests-gd-norot} 
\end{figure}

\end{document}